\documentclass[letterpaper]{article}
\date{}
\usepackage[margin=1in]{geometry}

\usepackage{fancyhdr}
\usepackage{color}
\usepackage{algorithm}
\usepackage{algorithmic}
\usepackage{natbib}

\usepackage{scrextend}
\usepackage{eso-pic} 
\usepackage{forloop}
\usepackage{tablefootnote}
\usepackage{algorithm}
\usepackage{algorithmic}
\usepackage{float}
\usepackage[utf8]{inputenc}
\usepackage{natbib}
\usepackage{scrextend}
\usepackage{eso-pic} 
\usepackage{forloop}
\usepackage{microtype}
\usepackage{graphicx}
\usepackage{subfigure}
\usepackage{tcolorbox}
\usepackage{booktabs} 
\usepackage[short]{optidef}
\usepackage{amsmath,amsthm,amssymb,amsfonts,bbm}
\usepackage{natbib}
 \newtheorem{theorem}{Theorem}
 \newtheorem{lemma}{Lemma}
 \newtheorem{corollary}{Corollary}
 \newtheorem*{proposition}{Proposition}

\DeclareMathOperator*{\argmin}{arg\,\,min}
\newcommand{\newparagraph}[1]{{\noindent\textbf{#1}~}} 
\newcommand{\newempar}[1]{{\noindent\emph{#1}~}} 
\usepackage[draft]{hyperref}
\usepackage{hyperref}

\title{Approximate Inference in Structured Instances \\ with Noisy Categorical Observations}

\begin{document}
\newtheorem{mydef}{Definition}
\newtheorem{example}{Example}
\maketitle
{\center
{\bf Alireza Heidari} \\
a5heidar@uwaterloo.ca\\
Department of Computer Science \\
University of Waterloo\\
{\bf Ihab F. Ilyas}  \\
ilyas@uwaterloo.ca\\
Department of Computer Science\\
University of Waterloo\\
{\bf Theodoros Rekatsinas}\\
thodrek@cs.wisc.edu \\
Department of Computer Science \\
University of Wisconsin-Madison \\
}

\begin{abstract}
We study the problem of recovering the latent ground truth labeling of a structured instance with categorical random variables in the presence of noisy observations. We present a new approximate algorithm for graphs with categorical variables that achieves low Hamming error in the presence of noisy vertex and edge observations. Our main result shows a logarithmic dependency of the Hamming error to the number of categories of the random variables. Our approach draws connections to correlation clustering with a fixed number of clusters. Our results generalize the works of ~\citet{globerson} and ~\citet{foster}, who study the hardness of structured prediction under binary labels, to the case of categorical labels.
\end{abstract}

\section{INTRODUCTION}
\label{sec:intro}

Statistical inference over structured instances of dependent variables (e.g., labeled sequences, trees, or general graphs) is a fundamental problem in many areas. Examples include computer vision~\citep{nowozin2011structured,dollar2013structured,chen2018deeplab}, natural language processing~\citep{huang2015bidirectional,hu2016harnessing}, and computational biology~\citep{li2007protein}. In many practical setups~\citep{shin2015incremental,holo, puds, heidari2019holodetect}, inference problems involve noisy observations of discrete labels assigned to the nodes and edges of a given structured instance and the goal is to infer a labeling of the vertices that achieves low disagreement rate between the correct ground truth labels $Y$ and the predicted labels $\hat{Y}$, i.e., low {\em Hamming error}. We refer to this problem as {\em statistical recovery}.

Our motivation to study the problem of statistical recovery stems from our recent work on data cleaning~\citep{holo, puds, heidari2019holodetect}. This work introduces HoloClean, a state-of-the-art inference engine for data curation that casts data cleaning as a structured prediction problem~\citep{puds}: Given a dataset as input, it associates each of its cells with a random variable, and uses logical integrity constraints over this dataset (e.g., key constraints or functional dependencies) to introduce dependencies over these random variables. The labels that each random variable can take are determined by the domain of the attribute associated with the corresponding cell. Since we focus on data cleaning, the input dataset corresponds to a noisy version of the latent, clean dataset. Our goal is to recover the latter. Hence, the initial value of each cell corresponds to a noisy observation of our target random variables. HoloClean employs approximate inference methods to solve this structured prediction problem. While its inference procedure comes with no rigorous guarantees, HoloClean achieves state-of-the-art results in practice. Our goal in this paper is to understand this phenomenon.

Recent works have also studied the problem of approximate inference in the presence of noisy vertex and edge observations. However, they are limited to the case of binary labeled variables: Globerson et al. focused on two-dimensional grid graphs and show that a polynomial time algorithm based on MaxCut can achieve optimal Hamming error for planar graphs for which a weak expansion property holds~\citep{globerson}. More recently, Foster et al. introduced an approximate inference algorithm based on tree decompositions that achieves low expected Hamming error for general graphs with bounded tree-width~\citep{foster}. In this paper, we generalize these results to the case of categorical labels.

\newparagraph{Problem and Challenges} We study the problem of statistical recovery over categorical data. We consider structured instances where each variable $u$ takes a ground truth label $Y_u$ in the discrete set $\{1,2, \dots, k\}$. We assume that for all variables $u$, we observe a noisy version $Z_u$ of its ground truth labeling such that $Z_u = Y_u$ with probability $1 - q$. We also assume that for all variable pairs $(u,v)$, we observe noisy measurements $X_{u,v}$ of the indicator $M_{u,v} = 2 \cdot \mathbbm{1}(Y_u = Y_v) - 1$ such that $X_{u,v} = M_{u,v}$ with probability $1-p$. Given these noisy measurements, our goal is to obtain a labeling $\hat{Y}$ of the variables such that the expected Hamming error between $Y$ and $\hat{Y}$ is minimized. We now provide some intuition on the challenges that categorical variables pose and why current approximate inference methods not applicable:  

First, in contrast to the binary case, negative edge measurements do not carry the same amount of information: Consider a simple uniform noise model. In the case of binary labels, observing an edge measurement $X_{u,v} = -1$ and a binary label $Z_u$ allows us to estimate that $\hat Y_v = -Z_u$ is correct with probability $(1 - q)(1-p) + qp$ when $p$ and $q$ are bounded away from 1/2. However, in the categorical setup, $\hat Y_v$ can take any of the $\{1,2, \dots, k\}\setminus \{Z_u\}$ labels, hence the probability of estimate $\hat Y_v$ being correct is up to a factor of $\frac{1}{k}$ smaller than the binary case. Our main insight is that while the binary case leverages edge labels for inference, approximate inference methods for categorical instances need to rely on the noisy node measurements and the positive edge measurements.

Second, existing approximate inference methods for statistical recovery~\citep{globerson, foster} rely on a ``Flipping Argument'' that is limited to binary variables to obtain low Hamming error: for binary node and edge observations, if all nodes in a maximal connected subgraph $S$ are labeled incorrectly with respect to the ground truth, then at least half of the edge observations on the boundary of $S$ are incorrect, or else the inference method would have flipped all node labels in $S$ to obtain a better solution with respect to the total Hamming error. As we discuss later, in the categorical case a naive extension implies that one needs to reason about all possible label permutations over the $k$ labels. 

\newparagraph{Contributions} We present a new approximate inference algorithm for statistical recovery with categorical variables. Our approach is inspired by that of~\cite{foster} but generalizes it to categorical variables. 

First, we show that, when a variable $u$ is assigned one of the $k-1$ erroneous labels with uniform probability $q/(k-1)$, the optimal Hamming error for trees with $n$ nodes is $\tilde{O}(\log(k)\cdot p \cdot n)$, when $q < 1/2$. This is obtained by solving a linear program using dynamic programming. Here, we derive a tight upper bound on the number of erroneous edge measurements, which we use to restrict the space of solutions explored by the linear program.

Second, we extend our method to general graphs using a tree decomposition of the structured input. We show how to combine our tree-based algorithm with correlation clustering over a fixed number of clusters~\citep{giotis2006correlation} to obtain a non-trivial error rate for graphs with bounded treewidth and a specified number of $k$ classes. Our method achieves an expected Hamming error of $\tilde{O}\big( k\cdot\log (k)\cdot p^{\lceil\frac{\Delta(G)}{2}\rceil}\cdot n\big)$ where $\Delta(G)$ is the maximum degree of graph $G$. We show that local pairwise label swaps are enough to obtain a globally consistent labeling with low expected Hamming error.

Finally, we validate our theoretical bounds via experiments on tree graphs and image data. Our empirical study demonstrates that our approximate inference algorithm achieve low Hamming error in practical scenarios.

\section{PRELIMINARIES}
\label{sec:prelims}

%

We introduce the problem of statistical recovery, and describe concepts, definitions, and notation used in the paper. We consider a structured instance represented by a graph $G = (V,E)$ with $|V| = n$ and $|E| = m$. Each vertex $u \in V$ represents a random variable with ground truth label $Y_u$ in the discrete set $L = \{1,2, \dots, k\}$. Edges in $E$ represent dependencies between random variables and each edge $(u,v) \in E$ has a ground truth measurement $M_{u,v} = \varphi(Y_u, Y_v)$ where $\varphi(Y_u,Y_v) = 1$ if $\mathbbm{1}(Y_u = Y_v) = 1$ and $\varphi(Y_u,Y_v) = -1$ otherwise.

\newparagraph{Uniform Noise Model and Hamming Error} We assume access to noisy observations over the nodes and edges of $G$. For each variable $u \in V$, we are given a noisy label observation $Z_u$, and for each edge $(u,v) \in E$ we are given a noisy edge observation $X_{u,v}$. These noisy observations are assumed to be generated from $G$, $Y$ and $M$ by the following process: We are given $G = (V,E)$ and two parameters, edge noise $p$ and node noise $q < 1/2$ with $p < q$. For each edge $(u,v) \in E$, the observation $X_{u,v}$ is independently sampled to be $X_{u,v} = M_{u,v}$ with probability $1-p$ (a {\em good} edge) and $X_{u,v} = -M_{u,v}$ with probability $p$ (a {\em bad} edge). For each node $u \in V$, the node observation $Z_u$ is independently sampled to be $Z_u = Y_u$ with probability $1-q$ (a {\em good} node) and can take any other label in $L \setminus Y_u$ with a uniform probability $\frac{q}{k-1}$. The uniform noise model is a direct extension of that considered by prior work~\citep{globerson, foster}, and a first natural step towards studying statistical recovery for categorical variables.

Given the noisy measurements $X$ and $Z$ over graph $G = (V,E)$, a labeling algorithm is a function $A$: $\{-1,+1\}^E \times \{1,2, \dots, k\}^V \rightarrow \{1,2, \dots, k\}^V$. We follow the setup of~\cite{globerson} to measure the performance of $A$. We consider the expectation of the Hamming error (i.e., the number of mispredicted labels) over the observation distribution induced by $Y$. We consider as error the worst-case (over the draw of $Y$) expected Hamming error, where the expectation is taken over the process generating the observations $X$ from $Y$. Our goal is to find an algorithm $A$ such that with high probability it yields bounded worst-case expected Hamming error. In the remainder of the paper, we will refer to the worst-case expected Hamming error as simply Hamming error.

\newparagraph{Categorical Labels and Edge Measurements} When $q$ is close to $0.5$, one needs to leverage the edge measurements to predict the node labels correctly. For binary labels, the structure of the graph $G$ alone determines if one can obtain algorithms with a small error for low constant edge noise $p$~\citep{globerson, foster}. We argue that this is not the case for categorical labels. Beyond the structure of the graph $G$, the number of labels $k$ also determines when we can obtain labeling algorithms with non-trivial error bounds.

We use the next example to provide some intuition on how $k$ affects the amount of information in the edge measurements of $G$: Let nodes take labels in $L = \{1, 2, \dots, k\}$. We fix a vertex $v$, and for each vertex $u$ in its neighborhood set the estimate label $\hat{Y}_u$ to $Z_{u}$ if $M_{u,v} = 1$ and to one of $L \setminus \{Z_u\}$ uniformly at random if $M_{u,v} = -1$. For a correct negative edge measurement and a correct label assignment to $v$, we are not guaranteed to obtain the correct label for $v$ as we would be able in the binary case. 

Given the above setup, the probability that node $u$ is labeled correctly is $P(\hat{Y}_u = Y_u) = (1 - b(1-\frac{1}{k-1}))\cdot ( (1-p)(1-q) + pq))$ where $b$ is the probability of an edge being negative in the ground truth labeling of $G$. Two observations emerge from this expression: (1) As the number of colors $k$ increases, the probability $P(\hat{Y}_u = Y_u)$ decreases, hence, for a fixed graph $G$ as $k$ increases, statistical recovery becomes harder; (2) For a fixed graph $G$, as $k$ increases the probability $b$ of obtaining a negative edge in the ground truth labeling of $G$ increases--- this holds for a fixed graph $G$ and under the assumption that each label should appear at least once in the ground truth---and the term $(1 - b(1-\frac{1}{k-1}))$ approaches zero. This implies that for $P(\hat{Y}_u = Y_u)$ to be meaningful the term $((1-p)(1-q) + pq)$ should be maximized for fixed $q$, and hence, the edge noise $p$ should approach zero as a function of $(1 - b(1-\frac{1}{k-1}))$. In other words, $p$ should be upper bounded by a function $\phi(k)$ such that as $k$ increases $\phi(k)$ goes to zero. We leverage these two observations to specify when statistical recovery is possible.
 
 \newparagraph{Statistical Recovery} Statistical recovery is possible for the family $\mathcal{G}$ of structured instances with $k$ categories, if there exists a function $f(p,k): [0,1] \rightarrow [0,1]$ with $\lim_{p \rightarrow 0} f(p,k) = 0$ such that for every $p$ that is upper bounded by a function $\phi(k)$ with $\lim_{k \rightarrow V} \phi(k) = 0$, the Hamming error of a labeling algorithm on graph $G \in \mathcal{G}$ with $V = n$ vertices is at most $f(p,k) \cdot n$.
 
 \begin{figure}
\begin{center}
\centerline{\includegraphics[width=0.8\columnwidth]{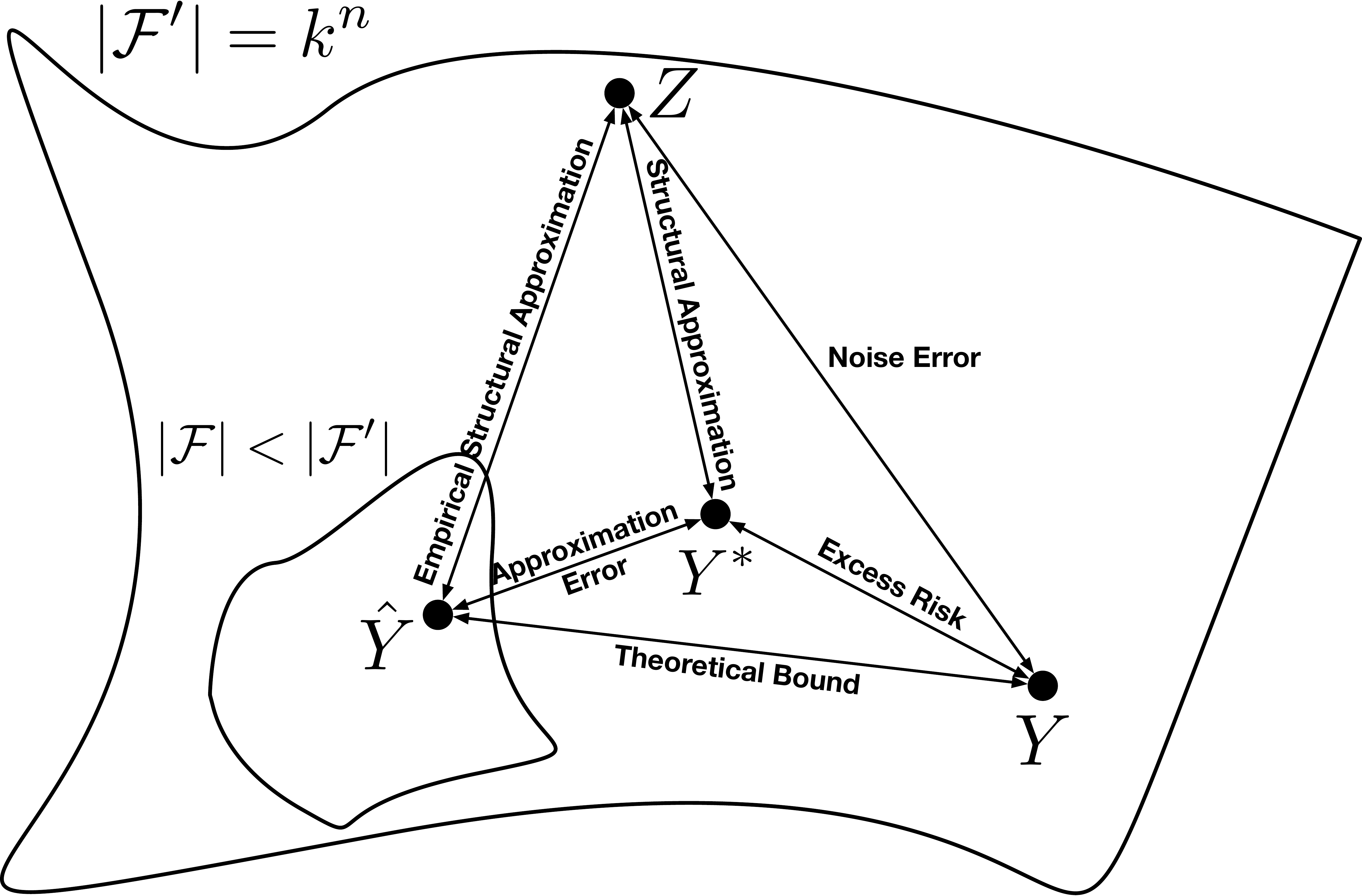}}
\vskip -0.1in
\caption{A schematic overview of our approach. Given the noise node labeling $Z$ of a graph with ground truth labeling $Y$, we leverage the noisy side information to obtain an approximate labeling $\hat{Y}$. Labeling $\hat{Y}$ is an approximate solution to the information theoretic optimal solution $Y^*$. The goal of our analysis is to find a theoretical bound on the Hamming error between $\hat{Y}$ and $Y$.}
\vskip -0.1in
\vspace{-20pt}
\label{fig:overview}
\end{center}
\end{figure} 

\section{APPROACH OVERVIEW}
\label{sec:analysis_overview}
We consider a graph $G=(V, E)$ with node labels in $L = \{1, 2, \dots, k\}$. The space of all possible labelings of $V$ defines a \emph{hypothesis space} $\mathcal{F}'$. In this space, we denote $Y$ the latent, ground truth labeling of $G$. In the absence of any information the size of this space is $|\mathcal{F}'|=k^{n}$. Access to any side information allows us to identify a subspace of $\mathcal{F}'$ that is close to $Y$.

First, we consider access only to noisy node labels of $G$ and denote $Z$ the point in $\mathcal{F}'$ for this labeling. If we have no side information on the edges of $G$, the information theoretic optimal solution to statistical recovery is $Z$ (because we assume $q < 1/2$). Second, we assume access only to edge measurements for $G$. We denote $X$ the observed edge measurements. If the edge measurements are accurate (i.e., $p = 0$) the size of $\mathcal{F}'$ reduces to $k!$. We assume that $k$ is such that one can obtain a labeling for $G$ that is edge-compatible with $X$ by traversing $G$. Under this assumption, the number of edge-compatible labelings is equal to all possible label permutations, i.e., $|\mathcal{F}'| = k!$. Finally, in the presence of both node and edge observations the information theoretic optimal solution to statistical recovery corresponds to a point $Y^*$ that is obtained by running exact marginal inference~\citep{globerson}. However, exact inference can be intractable, and even when it is efficient, it is not clear what is the optimal Hamming error that $Y^*$ yields with respect to $Y$.

To address these issues, we propose an approximate inference scheme and obtain a bound on the worst-case expected Hamming error that it obtains. We start with the noisy edge observations $X$ and use them to find a subspace $\mathcal{F} \subset \mathcal{F}'$ that contains node labelings which induce edge labelings that are close to $X$ (in terms of Hamming distance). We formalize this in the next two sections. Intuitively, we have that noisy edge measurements partition the space $\mathcal{F}$ in a collection of \emph{edge classes}.

\begin{mydef}
\label{edgeclass}
The edge class of a point $Y \in \mathcal{F}$ is a set $\mathbb{I} \in 2^{\{1, 2, \dots, k\}^{|V|}}$ such that for all $Y_i \in \mathbb{I}$, $Y_i$ induces the same edge measurements as $Y$. All points in $\mathbb{I}$ can be derived via a label permutation of $Y$. In general, for any labeling $Y'$, set $\mathbb{I}_{Y'} $ is the set of all labelings that can be generated by a label permutation of $Y'$. 
\end{mydef}
The restricted subspace $\mathcal{F}$ contains those edge classes that are close to the noisy edge observations $X$.

Given the restricted subspace $\mathcal{F}$, we design an algorithm to find a point $\hat{Y} \in \mathcal{F}$ such that the Hamming error between $\hat{Y}$ and $Y^*$ is minimized. We define the Hamming error with respect to an edge class $\mathbb{I}$ as:

\begin{mydef}
The Hamming error of a vector $\mathcal{Q} \in \{1,2, \dots, k\}^{|V|}$ to the edge class $\mathbb{I}_{Y'}\in 2^{\{1,2, \dots, k\}^{|V|}}$ is $Hd(\mathcal{Q} , \mathbb{I}_{Y'}) = \min_{\mathcal{Y} \in \mathbb{I}_{Y'} } Hd(\mathcal{Y},\mathcal{Q})$.
\end{mydef}

Point $Y^*$ might not be in $\mathcal{F}$ and the distance between $\hat{Y}$ and $Y^*$ is the approximation error we have due to approximate inference. Finally, we prove that the expected Hamming error between $\hat{Y}$ and $Z$ is bounded. A schematic diagram of our approximate inference method is shown in Figure~\ref{fig:overview}. In the following sections, we study statistical recovery for trees (in Section~\ref{sec:anal_trees}) and general graphs (in Section~\ref{sec:anal_graphs}). All proofs can be found in the supplementary material of our paper~\citep{uai}.

\section{RECOVERY IN TREES}
\label{sec:anal_trees}
We focus on trees and introduce a linear program for statistical recovery over $k$-categorical random variables. We prove that under a uniform noise model the optimal Hamming error is $\tilde{O}(\log(k)\cdot p \cdot n)$.

\subsection{A Linear Program for Statistical Recovery}
\label{sec:recovery_trees}

We follow the steps described in Section~\ref{sec:analysis_overview}. First, we use the noisy edge observations to restrict the search for $\hat{Y}$ to a subspace $\mathcal{F}$. We describe $\mathcal{F}$ via a constraint on the number of edge disagreements between the edge labeling implied by $\hat{Y}$ and the noisy edge observations $X$. Second, we form an optimization problem to find a point $\hat{Y}$ with minimum Hamming distance from $Z$ that satisfies the aforementioned constraint.

The ground truth edge labeling $M$ (corresponding to the ground truth node labeling $Y$) has bounded Hamming distance from the observed noisy labeling $X$. Hence, we can restrict the space of considered solutions to node labelings that induce an edge labeling with a bounded Hamming distance from the observed noisy labeling $X$. We have: Under the uniform noise model, edge measurements are flipped independently. Thus, the total number of bad edges is a sum over independent and identically distributed (iid) random variables. The expected number of flipped edges is $p\cdot|E| = p(n-1)$. Using the Bernstein inequality, we have:

\begin{lemma}
\label{treeedgebound}
Let $G$ be a graph with noisy edge observations with noise parameter $p$. With probability at least $1- \delta$ over the draw of $X$:
\begin{equation*}
\sum_{(u,v) \in E} \mathbbm{1}\{\varphi(Y_u,Y_v) \neq X_{u,v}\} \leq t \text{~~~where}
\end{equation*}   
\[t = (n-1)p + \frac{2}{3}\ln(\frac{2}{\delta})(1-p) + \sqrt{2(n-1)p(1-p)\ln(\frac{2}{\delta})}
\]
\end{lemma}
This lemma states that under the uniform noise model the ground truth edge labeling $M$ for Graph $G$ is in the neighborhood of $X$ with high probability. Given this bound, we use the following linear program to find $\hat{Y}$:

\begin{mini}
  [3]
  {\hat Y \in [k]^{|V|}}
  {\sum_{v \in V} \mathbbm{1}\{\hat Y_v \neq Z_v\}}
  {\label{lp_tree}}
  {}
  {}
  \addConstraint{\sum_{(u,v) \in E} \mathbbm{1}\{\varphi(\hat Y_u,\hat Y_v) \neq X_{u,v}\}}{ \leq t}
\end{mini}

where $t$ is defined as in Lemma~\ref{treeedgebound}. This problem can be solved via a dynamic programming algorithm with cost $O(k \cdot n^3 \cdot p)$. We describe this algorithm in the supplementary material of the paper~\citep{uai}.

\newparagraph{Discussion} Our approach is similar to that of~\cite{foster} for binary random variables. However, we use the Bernstein inequality to obtain a tighter concentration bound on the number of flipped edge measurements. In the case of categorical random variables, it is critical to obtain a tight description of the space $\mathcal{F}$ of the possible labeling solutions as we have a larger hypothesis space.

Let $\mathcal{S}(n,k)$ be the size of hypothesis space with $k$ labels and $n$ nodes.  If we increase $n$ by one, the rate of change for the hypothesis space is $r_{k,n}=\Delta\mathcal{S}/\Delta n=k^{n}(k-1)$, which is multiplicative with respect to $k$. Similarly, as we increase $k$ to $k+1$ the size of the hypothesis space changes by $s_{k,n}=\Delta\mathcal{S}/\Delta k=\sum_{i+j=n-1}(k+1)^ik^j\geq k^{n-1}$, which is exponential in the size of our input. We need a tight bound to obtain an efficient dynamic programming algorithm with respect to $n$ and $k$.

\subsection{Upper Bound on the Hamming Error for Trees}
\label{sec:stat_trees}

The Hamming error of $\hat{Y}$ obtained by Linear Program~\ref{lp_tree} is bounded by $\tilde{O}(\log(k)\cdot p\cdot n)$ with high probability. For our analysis, we draw connections to statistical learning.

We define a hypothesis class $\mathcal{F}$ that contains all points that satisfy the bound in Lemma~\ref{treeedgebound}: 
\[\mathcal{F} = \{Y' \in [k]^{|V|}: \sum_{(u,v) \in E}\mathbbm{1}\{\varphi( Y'_u, Y'_v) \neq X_{u,v}\} \leq t\}\]
From Lemma~\ref{treeedgebound}, we have that the edge class that corresponds to the ground truth labeling $Y$ is contained in $\mathcal{F}$ with high probability over the draw of $X$. Moreover, since the node noise $q$ is bounded away from $1/2$, we can use the noisy node measurements $Z$ to find a labeling $\hat{Y}$ that is in the same edge class as $Y$ and close to $Y$. Such a labeling is obtained by solving Linear Program~\ref{lp_tree}. From a statistical learning perspective, $\hat{Y}$ corresponds to the \emph{empirical risk minimizer} (ERM) over $\mathcal{F}$ given $Z$. Thus, the Hamming error between $\hat{Y}$ and $Y$ is associated with the \emph{excess risk} over $Z$ for Class $\mathcal{F}$. We have:

\begin{lemma} \label{ermlem} \citep{foster} Let $\hat{Y}$ be the empirical risk minimizer over $\mathcal{F}$ given $Z$ and let $Y^* = \argmin_{Y'\in \mathcal{F}} \sum\limits_{v\in V }\mathbb{P}(Y'_v \neq Y_v)$ and $c>0$ a constant number, then with probability $1-\delta$ over the draw of $Z$,
\begin{align*}
&\sum\limits_{v\in V}\mathbb{P}\big( \hat{Y}_v \neq Z_v\big)-\min\limits_{Y'\in \mathcal{F}} \sum\limits_{v\in V}\mathbb{P}\big( Y'_v \neq Z_v\big)\leq\\&  \bigg( \frac{2}{3}+ \frac{c}{2}\bigg)\log\bigg( \frac{|\mathcal{F}|}{\delta}\bigg)+\frac{1}{c}\sum\limits_{v\in V} \mathbbm{1}\big\lbrace \hat{Y}_v \neq Y^*_v \big\rbrace
\end{align*}
\label{erm}
\end{lemma}

We now analyze how the Hamming error relates to excess risk for categorical random variables. We have:

\begin{lemma}
\label{hamapproximate}
The Hamming error is proportional to the excess risk: For fixed $\hat Y, Y \sim  \mathcal{F}'$ and $Z$ distributed according to the uniform noise model we have that:
\[\mathbbm{1}\{\hat Y_v \neq Y_v\} = \frac{1}{c}\left[ P_Z(\hat Y_v \neq Z_v) - P_Z(Y_v \neq Z_v)\right] \]
\[\text{~where~~~} c = 1- k/(k-1)q\]
\end{lemma}
With $k = 2$ we have that $c = 1 - 2q$, which recovers the result of~\cite{foster} for binary random variables.

Using Lemma \ref{erm}, we can bound the excess risk in terms of the size of the hypothesis class. We have: 
\begin{corollary}
When $Y \in \mathcal{F}$ and $\hat Y = \arg\min_{\mathcal{Y}\in \mathcal{F}} \sum_{v \in V} \mathbbm{1}\{ \mathcal{Y}_v \neq Z_v\}$, we have that with probability at least $1 - \delta$ over the draw of $Z$:
\begin{gather*}
        \sum_{v \in V} P(\hat Y_v \neq Z_v) - \min_{Y' \in \mathcal{F}} \sum_{v \in V}  P(Y'_v \neq Z_v) \leq 
\left(\frac{4}{3} + \frac{2}{\frac{1}{4}+\big( \frac{1}{4}-\epsilon\big)\big( 1-\frac{k}{k-1}\big)}\right)\log\left(\frac{|\mathcal{F}|}{\delta}\right)
    \end{gather*}
\end{corollary}
We now combine these results with the complexity of class $\mathcal{F}$ to obtain a bound for the Hamming error:
\begin{theorem}
\label{treetheorem}
    Let $\hat Y$ be the solution to Problem~\ref{lp_tree}.
Then with probability at least $1 - \delta$ over the draw of $X$ and $Z$
    \begin{align*}
    \sum_{v \in V} \mathbbm{1}\{\hat Y_v \neq Y_v \}\leq
    & \frac{[t\log(2k)-\log(\delta)]}{\big( 1-\frac{k}{k-1}q\big)}  \left(\frac{4}{3} + \frac{2}{\frac{1}{4}+\big( \frac{1}{4}-\epsilon\big)\big( 1-\frac{k}{k-1}\big)}\right)\\ =&\tilde O(\log(k)np)
    \end{align*}
\end{theorem}
Here, $t$ is the same as in Lemma~\ref{treeedgebound}. We see that $k$ has a lower impact on the Hamming error than $n$ and $p$. Also, when $k = 2$ we recover the result of~\cite{foster}. Due to the tools we use to prove this result, this is a tight bound. We validate this bound empirically in Section~\ref{sec:exps}.

\section{RECOVERY IN GENERAL GRAPHS}
\label{sec:anal_graphs}

We now show how our tree-based algorithm can be combined with correlation clustering to obtain a non-trivial error rate for graphs with bounded treewidth and $k$-categorical random variables. We first describe our approximate inference algorithm and then show that our algorithm achieves an expected Hamming error of $\tilde{O}\big( k\cdot\log (k)\cdot p^{\lceil\frac{\Delta(G)}{2}\rceil}\cdot n\big)$ where $\Delta(G)$ is the maximum degree of the structured instance $G$.

\subsection{Approximate Statistical Recovery}
\label{sec:cc_algorithm}                                                                                                                                                                         

We build upon the concept of \emph{tree decompositions}~\citep{diestel2018graph}. Let $G$ be a graph, $T$ be a tree, and $\mathcal{W} = (V_t)_{t\in T}$ be a family of vertex sets $V_t \subseteq V(G)$ indexed by the nodes $t$ of $T$. We denote a tree-decomposition with $(T,\mathcal{W} )$. The width of $(T, \mathcal{W})$ is defined as $\max\{|V_t|-1:t\in T \}$ and the treewidth $tw(G)$ of $G$ is the minimum width among all possible decompositions of $G$. We also denote with $F$ the $|\mathcal{W}|-1$ edges connecting the bags in $\mathcal{W}$ in $(T, \mathcal{W})$ and represent $T$ as $T=(\mathcal{W},F)$.

Given a graph $G$, a tree decomposition of $T$ defines a series of local subproblems whose solutions can be combined via dynamic programming to obtain a global solution for the original problem on $G$. For graphs of bounded treewidth, this approach allows us to obtain efficient algorithms~\citep{bodlaender88}. Our solution proceeds as follows: Let $(T, \mathcal{W})$ be a tree decomposition of $G$. We first find a {\em local} labeling $\tilde{Y}^W$ for each $W \in \mathcal{W}$. Then, we design a dynamic programming algorithm that combines all local labelings to obtain a global labeling $\hat{Y}$.

\subsubsection{Finding Local Labelings}
\label{sec:localopt}
We recover the labeling of the nodes in a bag $W$ as follows: (1) Given $W$, we consider a superset of $W$, defined as $W^* = EXT(W) = W \cup \big( \bigcup_{v\in G} N(v)\big)$ where $N(v)$ is the one-hop neighborhood of node $v$; (2) Given $W^*$, we use the edge observations in the edge subset $E' \subseteq E$ induced by $W^*$ to find a restricted hypothesis space $\mathcal{F}_{W^*}$. We then find a labeling $\tilde{Y}^{W^*} \in \mathcal{F}_{W^*}$ that has the minimum Hamming error with respect to $Z$ for the nodes in $W^*$. Let $Z_{W^*}$ denote this subset of $Z$; (3) For $W$, we assign $\tilde Y^W$ to be the restriction of $\tilde{Y}^{W^*}$ on $W$.

We consider two cases for Step 2 from above: (1) If $|W^*| = O(\log(n))$, we can enumerate all $k^{O(\log(n))}$ labelings for $W^*$ and choose the one with minimum Hamming distance from $Z$. The complexity of this brute-force algorithm is $k^{O(\log(n))} = \mathtt{poly}(n)$; (2) If $|W^*| = \Omega(\log(n))$, we use the \textsc{MaxAgree[$k$]} algorithm of~ \cite{giotis2006correlation} over the noisy edge measurements $X$ to restring the subspace $\mathcal{F}$ in the neighborhood of $X$. \textsc{MaxAgree[$k$]} is a polynomial-time approximation scheme (PTAS) for solving the Max-Agreement version of correlation clustering for a fixed number of $k$ labels. In the worst case, \textsc{MaxAgree[$k$]} obtains an approximation of 0.7666\textsc{Opt[$k$]}. In our analysis, we account for the approximation factor 0.7666 by changing the probability $p$ to $p' = 0.7666p + 0.2334$. A detailed discussion is provided in the supplementary material of the paper~\citep{uai}. Given the output of \textsc{MaxAgree[$k$]}, let $\mathcal{F}_{CC}$ be the restricted subspace of solutions for $W^*$. We pick an arbitrary labeling $\bar{Y}^{W^*} \in \mathcal{F}_{CC}$ and use Algorithm~\ref{alg:greedy} to get a permutation that transforms $\bar{Y}^{W^*}$ to point $\tilde{Y}^{W^*}$ that has minimum Hamming distance to $Z^{W^*}$.

\begin{algorithm}
\small
    \renewcommand{\algorithmicrequire}{\textbf{Input:}}
    \caption{Local Label Permutation}
    \begin{algorithmic}
            \REQUIRE  A labeling $\bar Y^{W^*}$ in the subspace $\mathcal{F}_{CC}$ identified by \textsc{MaxAgree[$k$]} on ${W^*}$; Node observations $Z^{W^*}$;
            \STATE $\bar Y^{W^*}_1, \bar Y^{W^*}_2, \ldots \bar Y^{W^*}_k \leftarrow$ Group $\bar Y^{W^*}$ By Label;
            \STATE $Z^{W^*}_1, Z^{W^*}_2, \ldots Z^{W^*}_k \leftarrow$ Group $Z^{W^*}$ By Label;                      
            \FOR{$i,j\in [k]\times [k]$}
            \STATE $I_{i,j} \leftarrow |\bar{Y}^{W^*}_i \cup Z^{W^*}_j|$;
            \ENDFOR
            
            \STATE $Q \leftarrow$ A queue that sorts $I = \{I_{i,j}\}_{(i,j) \in [k]\times [k]}$ in decreasing order with respect to values $I_{i,j}$;
            \WHILE{$Q\neq \emptyset$}
            \STATE $I_{i,j} \leftarrow$ Pop($Q$);
            \STATE $\pi(i) \leftarrow j$;
            \STATE Remove all $I_{t,j}$ and $I_{i,t}$ for all $t\in [k]$ from $Q$;
            \ENDWHILE
            \STATE \textbf{Return:} $\pi$
    \end{algorithmic}
    \label{alg:greedy}
\end{algorithm}

Algorithm~\ref{alg:greedy} greedily permutes the labels in $\bar Y^w$ to obtain a labeling with minimum Hamming distance to $Z^W$. The complexity of this algorithm is $O(n+k\log k)$.

\begin{lemma}
\label{greedyresult}
    Algorithm \ref{alg:greedy} finds a permutation $\pi$ such that: 
    \begin{align*}
        \tilde Y^W=\pi(\bar Y^W) = \min_{\pi\in\Gamma_k} \sum_{v \in W} \mathbbm{1}\{\pi(\bar Y^W) \neq Z^W\}
    \end{align*}
    where $\Gamma_k$ is the set of all permutations of the $k$ labels.
\end{lemma}

We combine all steps in Algorithm~\ref{alg:local}. The output of this algorithm is a collection of labelings $\tilde Y$ for the local problems. Lemma~\ref{greedyresult} states that $\tilde Y^{W^*}$ minimizes the Hamming distance to $Z$. We also show that $\tilde Y^{W^*}$ remains a minimizer with respect to $\min_{y}\sum_{(u,v)} \mathbbm{1}(\varphi(y_u,y_v)\neq X_{uv})$ after the swaps due to $\pi$.

\begin{algorithm}
\small
    \renewcommand{\algorithmicrequire}{\textbf{Input:}}
    \caption{Find Local Labelings}
    \begin{algorithmic}
            \REQUIRE  A tree decomposition $T = (\mathcal{W}, F)$ of $G$; Noisy node observations $Z$; Noisy edge measurements $X$;
            \STATE $\tilde Y \rightarrow \emptyset$;
            \FOR{$W \in \mathcal{W}$}
	        \STATE $W^* = EXT(W)$;
	        \STATE $ \backslash* $ The next optimization problem can be solved either via enumeration or correlation clustering. $E(W^*)$ denotes the set of edges in $W^*$.$ *\backslash $
             \STATE $\bar{Y}^{W^*}= \arg\min\limits_{y}\sum\limits_{(u,v)\in E(W^*)} \mathbbm{1}\{\varphi(y_u,y_v)\neq X_{uv}\} $;
            \STATE $\tilde{Y}^{W^*} \leftarrow$ Local Label Permutation $(\bar{Y}^{W^*}, Z^{W^*}) $;
            \STATE  Let $\tilde{Y}^{W}$ be the restriction of $\tilde{Y}^{W^*}$ to $W$;
            \STATE  $\tilde Y \rightarrow \tilde Y \cup \{\tilde Y ^{W}\}$;
        	\ENDFOR
            \STATE \textbf{Return:} $\tilde Y$
    \end{algorithmic}
    \label{alg:local}
\end{algorithm}

\begin{mydef}
Given a graph $G=(V,E)$, the $swap(V,c_1,c_2)$ function changes all node labels $c_1$ to $c_2$, and all node labels $c_2$ to $c_1$.
\end{mydef}
The swap operation enables us to switch between elements within an edge class. We show that a $swap(V,c_1,c_2)$ does not affect the disagreements between the node labeling and edge labeling of a graph.

\begin{lemma}
\label{lem:swap}
Let $L$ be a set of labels $L = \{1, 2, \dots, k\}$. Consider a graph $G = (V,E)$ for which we are given a node labeling $Y$ and an edge labeling $X$. For any pair $(c,c') \in L \times L$, let $Y' = swap(V, c, c')$ be the node labeling of $G$ after swapping label $c$ with $c'$. We have that:
	$\sum_{(u,v) \in E} \mathbbm{1}\{\varphi(Y_u,Y_v) \neq X_{u,v}\} = \sum_{(u,v) \in E} \mathbbm{1}\{\varphi(Y'_u,Y'_v) \neq X_{u,v}\}$.
\end{lemma}
This lemma implies that $\tilde Y^{W^*}$ is a minimizer of $\min_{y}\sum_{(u,v)} \mathbbm{1}\{\varphi(y_u,y_v)\neq X_{uv}\}$ since $\bar Y ^{W^*}$ minimizes this quantity, and $\tilde Y^{W^*}$ is a permutation of $\bar Y ^{W^*}$.

\subsubsection{From Local Labelings to a Global Labeling}
\label{sec:globalopt}

We now describe how to combine labelings $\{\tilde Y^W\}_{W \in \mathcal{W}}$ into a global labeling $\hat Y$. For binary random variables, the following procedure plays a central role in enforcing agreement across local labelings~\citep{foster}: Given a bag $W_1$ and a neighbor $W_2$ with conflicting node labels with respect to $W_1$, we can maximize the agreement between $W_1$ and $W_2$ by flipping labeling $\tilde Y^{W_1}$ to its mirror labeling. This operation leads to consistent solutions since for binary random variables there is only one mirror labeling. However, for categorical random variables we have $k!$ possible mirror labelings for $\tilde Y^{W_1}$. {\em We show that it suffices to consider only one label swap per bag instead of $k!$ labelings.}

We consider the swap operation (see Section~\ref{sec:localopt}) and two bags $W_1$ and $W_2$ with labelings $\tilde Y^{W_1}$ and $\tilde Y^{W_2}$. We resolve conflicts in $W_1 \cap W_2$ as follows: Let $\Pi_k \subset \Gamma_k$ be the set of all permutations restricted to one pairwise color swap. Given a bag $W \in \mathcal{W}$ with labeling $Y^W$, we define a swap $\pi = swap(W, c_i, c_j)$ to be valid if color $c_i$ is present in $Y_W$. Given a valid swap $\pi$ for $W$, we define $\pi(Y^W)$ to be the label assignment for all nodes in $W$ after applying $\pi$ to $Y^W$. Also, let $\pi(Y_v^W)$ be the labeling for a node $v \in W$ after $\pi$. Finally, we define $\Pi_k(Y^W)$ as the set of all labelings for $W$ that can be obtained if we apply any valid pairwise label swap on $Y^W$. To resolve inconsistencies between $\tilde Y^{W_1}$ and $\tilde Y^{W_2}$, we consider pairs in $\Pi_k(Y^{W_1}) \times \Pi_k(Y^{W_2})$ such that the labeling in the intersection of $W_1$ and $W_2$ is consistent and the number of nodes whose label is swapped is minimum.


The procedure we use is shown in Algorithm~\ref{alg:global}. The algorithm takes as input a tree decomposition $T = (\mathcal{W}, F)$ of $G$ and the local labelings $\tilde Y$. For each $W$ with labeling $\tilde Y^{W}$, we compute the cost of swapping label $c_i$ with label $c_j$ for each $(i,j) \in [k] \times [k]$. Then, we iterate over edges in $F$ to identify incompatibilities between local node labelings. Finally, we use all the computed costs to find the single swap $\pi_W$ to be applied locally to each bag $W \in \mathcal{W}$ such that global agreement is maximized. To this end, we solve a linear program similar to program~\ref{lp_tree}. This program is shown in Algorithm~\ref{alg:catTreeDec}. 

\begin{algorithm}
\small
    \renewcommand{\algorithmicrequire}{\textbf{Input:}}
    \caption{From Local Labelings to a Global Labeling}
    \begin{algorithmic}
            \REQUIRE  A tree decomposition $T = (\mathcal{W}, F)$ of $G$; Noisy node observations $Z$; Noisy edge measurements $X$; Local labelings $\{\tilde Y^{W}\}_{W \in \mathcal{W}}$;
            \STATE $\hat Y \rightarrow \emptyset$;
			\FOR{$W \in \mathcal{W}$}
				\STATE $\Pi^W_k \leftarrow$ the set of valid pairwise color swaps for $W$;
         		\FOR{  $\pi \in \Pi^W_k$}
		         \STATE  $ \backslash* $ $\pi$ is associated with a label swap $(c_i, c_j)$ $*\backslash$;
        		\STATE $Cost_W [\pi]= \sum\limits_{v\in W}  \mathbbm{1}(\pi(\tilde{Y}^W)\neq Z^W) $;
		        \ENDFOR
        	\ENDFOR
	       \FOR{$(W_1, W_2)\in F$}
            \STATE Select one node $v$ from $W_1\cap W_2$ randomly;
            \STATE $S(W_1,W_2)=2\cdot\mathbbm{1}\{\tilde {Y}_v^{W_1}=\tilde{Y}_v^{W_2}\}-1 $;
	        \ENDFOR
	        \STATE Compute constant $L_n$; $ \backslash* $ See Section~\ref{sec:mainbound} $ *\backslash$;
            \STATE $\{\pi_{W}\}_{W \in \mathcal{W}}= $~Cat. Tree Decoder$(T, Cost, S, L_n)$;
            \FOR{$v\in V$}
            \STATE  Choose arbitrary $W$ s.t. $v\in W$ randomly;
            \STATE  $\hat{Y}_v= \pi_W(\tilde{Y}^W_v)$
	        \ENDFOR
       	   \STATE \textbf{Return:} $\hat{Y}$
    \end{algorithmic}
    \label{alg:global}
\end{algorithm}

In Algorithm~\ref{alg:catTreeDec}, function $\psi(\cdot)$ is defined as:
\begin{align*}
& \psi(\pi_W,\pi_{W'})
=\begin{cases}
       1,\text{~if~}\pi_W(\tilde Y^W_v)=\pi_{W'}(\tilde Y^{W'}_v) : \forall v \in W \cap W' \\
        -1,\text{~if~}\pi_W(\tilde Y^W_v)\neq\pi_{W'}(\tilde Y^{W'}_v) : \exists v \in W \cap W' 
     \end{cases}
\end{align*}
Constant $L_n$ is used to restrict the space of solutions considered. A discussion on $L_n$ is deferred to Section~\ref{sec:mainbound}.

\begin{algorithm}[h]
\small
\renewcommand{\algorithmicrequire}{\textbf{Input:}}
\renewcommand{\algorithmicensure}{\textbf{Output:}}
\caption{Categorical Tree Decoder}
    \begin{algorithmic}
    \label{alg:catTreeDec}
        \REQUIRE A tree $T= (\mathcal{W},F)$; Matrices $\lbrace Cost_W\rbrace_{W\in \mathcal{W}}$, $\lbrace S(W,W')\rbrace_{(W,W')\in F}$, $L_n\in \mathbb{N}$;
        \ENSURE Optimal swaps $\{\pi_{W}\}_{W \in \mathcal{W}}$ for each $W \in \mathcal{W}$;
        \STATE Solve the linear program:     
        \STATE $\hat{\Pi} = \arg\min \limits_{\{\pi_W\}_{W \in \mathcal{W}} \in \Pi_k^{|\mathcal{W}|}} \sum\limits_{W \in \mathcal{W}}Cost_W[\pi_W]$
        \STATE s.t. $\sum\limits_{(W,W')\in F} \mathbbm{1}\{\psi(\pi_W,\pi_{W'}\}\neq S(W,W')\leq L_n$ 
         \STATE \textbf{Return:} $\hat{\Pi}$ 
    \end{algorithmic}
\end{algorithm}

\subsubsection{Discussion on Correlation Clustering}
\label{sec:discussion}

We use correlation clustering in our algorithm for practical reasons. If the cardinality of the bags $T = (\mathcal{W}, F)$ is bounded by $O(\log(n))$, we can find a local labeling for each $W$ that has minimum Hamming distance to $Z$ efficiently. Obtaining such a decomposition $T$ is an NP-complete problem. This challenge is also highlighted by~\cite{foster}. To address this issue they assume a sampling procedure for removing edges from $G$ to obtain a subgraph for which a low-width tree decomposition is easy to find. This procedure is a graph-specific exercise and not easily generalizable to arbitrary graphs. We follow a different approach. Instead of using specialized procedures, we rely on heuristics to obtain a low-width decompositions~\cite{deGivry, Dermaku} and use correlation clustering for large bags. This scheme allows us to use our algorithm with arbitrary graphs.

\subsection{A Bound for Low Treewidth Graphs}
\label{sec:mainbound}
We state our main theorem for statistical recovery over general graphs. We also provide a proof sketch.

\begin{theorem}
\label{maintheorem}
\emph{(\textbf{Main Theorem})} Consider graph $G$ with $T = (\mathcal{W}, F)$, noisy node observations $Y$, and noisy edge observations $X$. Let $\hat Y$ be the statistical recovery solution obtained by combining Algorithms~\ref{alg:local} and~\ref{alg:global}. With high probability over the draw of $Z$ and $X$: 
\begin{align*}
\sum\limits_{v\in V} \mathbbm{1}\big\lbrace \hat{Y}_v \neq Y_v\big\rbrace &\leq \tilde{O}\big( k \cdot \log k\cdot p^{\lceil\frac{\mathtt{mincut^*}(G)}{2}\rceil} \cdot n\big)\\& \leq  \tilde{O}\big( k \cdot \log k \cdot p^{\lceil\frac{\Delta}{2}\rceil} \cdot n\big)
\end{align*}
where $mincut^*(G)$ is the min. mincut over all extended bags in $\mathcal{W}$ and $\Delta(G)$ is the max. degree in $G$.
\end{theorem}
We see that the Hamming error obtained by our approach goes to zero as $p \rightarrow 0$. Theorem~\ref{maintheorem} allows us to understand when statistical recovery over a graph with categorical random variables is possible (i.e., when we can rely on edge observations to solve statistical recovery more accurately than the trivial solution of keeping the initially assigned node labels). Theorem~\ref{maintheorem} connects the level of edge-noise with the degree $\Delta$ of the input graph, the number of labels $k$, and the noise q on node labels. We have that for the edge noise $p$ it should be $p \leq \sqrt[\leftroot{-2}\uproot{5}\lceil\frac{\Delta}{2}\rceil]{\frac{q}{k\log k}}$, where $q$ is the node noise parameter, for the side information in $X$ to be useful for statistical recovery. Otherwise, one should just use the initially observed node labels.

\newparagraph{Proof Sketch} Let $S$ denote a maximal connected subgraph of $G$. Let $\delta(S)$ be the boundary of $S$, i.e., the set of edges with exactly one endpoint in $S$. Let $\tilde Y^S$ be the local labeling for nodes in $S$. We say that $S$ is incorrectly labeled if for all $v \in S$ we have $\tilde Y^S_v \neq Y_v$. We have:

\begin{lemma}
\label{swappinglemma}
(Swapping lemma) Let $S$ be a maximal connected subgraph of $G$ with every node incorrectly labelled by $\tilde Y$. Then at least half the edges of $\delta(S)$ are bad.
\end{lemma}

For a bag $W$, let set $S$ be the largest connected component in $W$ such that for all nodes $v$ in it $\tilde Y^W_v \neq Y_v$. It must be the case that at least half of the $\delta(S)$ edges are incorrect or else there exists a different labeling that agrees with $X$ better than $\tilde Y^W$. This contradicts the fact that $\tilde Y^W$ is a minimizer of $\min\limits_{y}\sum_{(u,v)} \mathbbm{1}\{\varphi(y_u,y_v)\neq X_{uv}\}$. This result extends the Flipping Lemma of ~\cite{globerson} from the binary to the categorical case.

We use this result to bound the probability that a local labeling $\tilde Y^W$ (see Lemma~\ref{greedyresult}) will fail to recover the ground truth node label for $W$. The probability of local labelings having large Hamming error is upper bounded: 
\begin{lemma}
\label{expectation}
Let $\Gamma_k$ be the all label permutations on the set $L = \{1, 2, \dots, k\}$. We have for $W$:
\begin{align*}
\mathbb{P}\bigg(\min\limits_{\pi \in \Gamma_k} \mathbbm{1}\lbrace \pi(\bar{Y}^{W})\neq Y^{W}\rbrace& > 0\bigg)\leq  2^{|W^*|}p^{\lceil\frac{\texttt{mincut}^*(W)}{2}\rceil}
\end{align*}
with $mincut^*(W)=\min\limits_{S\subset W^*,S\cap W\neq \emptyset, \bar{S}\cap W\neq \emptyset } |\delta_{G(W)}(S)|$.
\end{lemma}

We now build upon Lemma~\ref{expectation} and leverage the result introduced by~\cite{boucheron2003concentration} to obtain an upper bound on the total number of mislabeled nodes across all bags in $\mathcal{W}$ for any labeling permutation $\pi \in \Gamma_k$ over the local labeling $\tilde{Y}^W$:

\begin{lemma}
\label{globalbound}
Let $\Gamma_k$ be the all label permutations on the set $L = \{1, 2, \dots, k\}$. For all $\delta>0$, with probability at least $1-\frac{\delta}{2}$ over the draw of $X$ we have that:
\begin{align*}
&\min\limits_{\pi \in \Gamma_k} \sum\limits_{W\in \mathcal{W}} \mathbbm{1}\lbrace \pi(\tilde{Y}^W) \neq Y^W\rbrace \leq 2^{|W|+1}p^{\lceil\frac{\mathtt{mincut}(W)}{2}\rceil}+6\max\limits_{e\in E}|\mathcal{W}(e)|\max\limits_{W\in \mathcal{W}}|E(W)|\log(\frac{2}{\delta})
\end{align*}
where $\mathcal{W}(e)$ denotes the set of bags in $\mathcal{W}$ that contain edge $e$ and $E(W)$ denotes the set of edges in bag $W$.
\end{lemma}
This lemma can be extended to $W^*$ as well. This lemma combined with Lemma~\ref{swappinglemma} implies that the labeling disagreement across bags in the tree decomposition are bounded. The analysis continues in a way similar to that for trees (see Section~\ref{sec:anal_trees}). Given the local bag labelings, we seek to find the labeling swaps across bags such that the global labeling has minimum Hamming error with respect to $Y$. We use the inequality from Lemma~\ref{globalbound} to restrict the space $([k] \times [k])^{\mathcal{W}}$ of all possible pairwise label swaps over the local bag labelings. Let $s^*$ be the optimal point in $([k] \times [k])^{\mathcal{W}}$ such that the global labeling has minimum Hamming error with respect to $Y$. Given the tree decomposition $T = (\mathcal{W}, F)$ of $G$. We define the hypothesis space:
\begin{align*}
&\mathcal{F}\triangleq ([k] \times [k])^{\mathcal{W}} \\&\text{~~~s.t.}\sum\limits_{(W,W')\in F} \mathbbm{1}\{\psi(\pi_W,\pi_{W'})\neq S(W,W')\}\leq L_n\big\rbrace
\end{align*}
with $L_n = deg(T)\big[ 2^{wid^*(W)+2}\sum_{W\in \mathcal{W}}p^{\lceil\frac{mincut^*(W)}{2}\rceil}+6deg_E^*(T)\max_{W\in \mathcal{W}}|E(W^*)|\log(\frac{2}{\delta})\big]$, $deg^*_E(T) = \max_{e\in E} |\mathcal{W}(e)|$, and $\pi_W$ and $S(W,W')$ denote the pairwise swaps and labeling disagreements between bags from Algorithm~\ref{alg:global}. We show that the optimal permutation $\Pi^*$ is a member of $\mathcal{F}$ with high probability and also have that $|\mathcal{F}(X)| \leq  \big( \frac{e\cdot n \cdot k!}{L_n}\big)^{L_n}$. Combining this with Lemma \ref{ermlem}, we take $\hat{\Pi}$ is most correlated with $Z$, i.e., it is a minimizer for \( \sum_{W\in \mathcal{W}}\sum_{v\in W} \mathbbm{1}\big\lbrace \pi_W(\tilde{Y}_v^{W}) \neq Z_v\big\rbrace\). Directly from statistical learning theory we have that the Hamming error of this estimator $\hat Y$ is $\tilde O(\log(\mathcal{F})) = \tilde{O}\big( k \cdot \log k \cdot p^{\lceil\frac{\Delta}{2}\rceil} \cdot n\big)$ which establishes our main theorem.

\section{EXPERIMENTS}
\label{sec:exps}
\newparagraph{Experimental Setup} We evaluate our approach on trees and grid graphs. For trees, we use Erdős–Rényi random trees to obtain ground truth instances. For grids, we use real images to obtain the ground truth. We create noisy observations via a uniform noise model. We compare our approach with two approximate inference baselines: (1) a Majority Vote algorithm, where we leverage the neighborhood of a node to predict its label, and (2) (Loopy) Belief Propagation. To evaluate performance we use the normalized Hamming distance $ \sum_{v\in V} \mathbbm{1}(Y_v\neq \hat Y_v)/|V|$. We provide more details in the Supplementary Material.

\newparagraph{Hamming Error of Random Trees} Our analysis suggests that Linear Program \ref{lp_tree} yields a solution with Hamming error $\tilde O(\log(k)np)$. We evaluate experimentally that the Hamming error increases at a logarithmic rate with respect to $k$. Figure \ref{treelogarithmic} shows the Hamming error for a fixed tree generative model with $p=0.1$ and $q = 0.2$ as we increase the number of labels $k$. We fix $q$ away from $0.5$ and generate $10,000$ trees for each $k$. We report the average error. As shown, we observe the expected logarithmic behavior that we proved theoretically. The graph size is chosen randomly $n\in [10^3, 1.5\times 10^3]$.

\begin{figure}[ht]
\centering
\includegraphics[width=0.8\columnwidth]{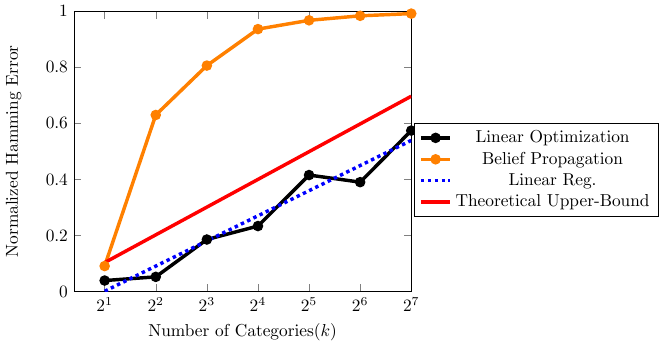}
\caption{Experimental validation that Hamming error for trees increases with a logarithmic rate w.r.t. $k$.}
\label{treelogarithmic}
\end{figure}

\newparagraph{Hamming Error of Grids} We have two experiments on grids. In the first experiment, we select $1,000$ grayscale images and compute the Hamming error obtained by our algorithm. We consider a uniform noise model with $p=0.05$ and $q=0.1$. Figure \ref{grid-comp} shows the Hamming error as $k$ increases. As expected we see that the Hamming error increases. This is because as $k$ increases negative edges carry lower information, and with non-zero edge error (p), the positive edges also provide low information observations (i.e., a wrong measurement). In the supplementary material of our paper, we present a qualitative evaluation of our results on the grey-scale images.

\begin{figure}[ht]
\centering
\includegraphics[width=0.8\columnwidth]{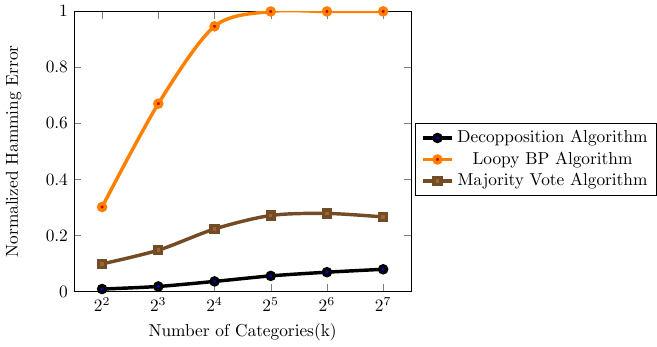}
\vskip -0.15in
\caption{The Hamming error for different methods on grids. We show mean the mean error of $1,000$ repetitions.}
\label{grid-comp}
\vskip -0.1in
\end{figure}

In the second experiment, we evaluate the effect of edge noise $p$ on the quality of solution obtained by our methods for a fixed number of labels $k$ and fixed node noise $q$. In Figure \ref{change_p}, we show the effect of $p$ on the average of Hamming error when other parameters are fixed ($n=6\times 10^4, k=128, q=0.1$). We vary $p$ from zero to $0.5$. We repeat each experiment $100$ times. We find that our approximate inference algorithm is robust to small amounts of noise.

This experiment also validates Theorem~\ref{maintheorem} which states when the side information from edges $X$ helps with statistical recovery. For the setups we consider in this experiment, we have k = 128 and vary q in {0.1, 0.15, 0.2}. If we keep the initial node labels the expected normalized Hamming error will be 0.1, 0.15, and 0.2 respectively. Theorem~\ref{maintheorem} states that to obtain a better Hamming error than the above one, the edge noise $p$ has to be less than $\sqrt{0.1/(128\log128)} \sim 0.04$, $\sqrt{0.15/(128\log128)} \sim 0.05$, $\sqrt{0.2/(128\log128)} \sim 0.06$ respectively. Figure~\ref{change_p} shows that the normalized Hamming error obtained by our algorithm reaches the Hamming error of the trivial algorithm (and plateaus around it) at the expected edge-noise levels of 0.04, 0.05, and 0.06.

\begin{figure}[ht]
\centering
\includegraphics[width=0.5\columnwidth]{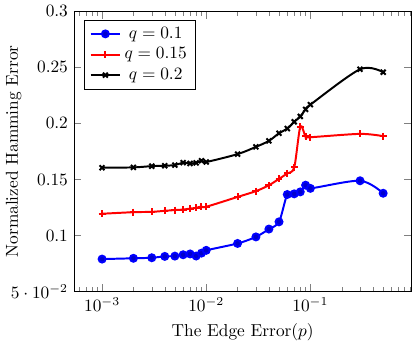}
\vskip -0.1in
\caption{The effect of varying $p$ on the average of normalized Hamming error($Hd$) with fixed $q$.}
\vskip -0.1in
\label{change_p}
\end{figure}

Our approximate inference algorithm is robust to small amounts of noise. As expected, when the noise increases the Hamming error increases.

\section{CONCLUSION}
\label{sec:conclusion}
We considered the problem of statistical recovery in structured instances with noisy categorical observations. We presented a new approximate algorithm for inference over graphs with categorical random variables. We showed a logarithmic dependency of the Hamming error to the number of categories the random variables can obtain. We also explored the connections between approximate inference and correlation clustering with a fixed number of clusters. There are several future directions suggested by this work. One interesting direction would be to understand under which noise models the problem of statistical recovery is solvable. Moreover, it is interesting to explore the direction of correlation clustering further and extend our analysis beyond small tree width graphs.

\vspace{-10pt}\paragraph{Acknowledgements} The authors thank Shai Ben David, Fereshte Heidari Khazaei, and Joshua McGrath for the helpful discussions. This work was supported by Amazon under an ARA Award, by NSERC under a Discovery Grant, and by NSF under grant IIS-1755676.





\section{ANALYSIS FOR TREES}
\subsection{Proof of Lemma 1}
\begin{proof}
In $G=(V, E)$, for each edge $(u,v) \in E$, we have a random variable $L_{u,v} = \mathbbm{1}(\varphi(Z_u,Z_v) \neq X_{u,v})$ with distribution:
\[L_{u,v} = \begin{cases}
1, & p\\
0, & 1-p
\end{cases}\]
To apply the Bernstein inequality, we must consider $L_{u,v} - p$. We have $E[L_{u,v} - p] = 0$ and $\sigma^2(L_{u,v} - p) = p(1-p)$. We must also have that the random variables are constrained. We know that
$|L_{u,v} - p| \leq \max\{1-p, p\}$ and $p < 1/2$ so $|L_{u,v} - p| \leq 1-p$. Now, we apply the Bernstein inequality:

\begin{align*}
   &P\left(\sum_{(u,v) \in E} L_{u,v} - p \leq t\right) \geq 1 - \exp\left(-\frac{t^2}{2|E|\sigma^2 + \frac{2}{3}(1-p)t }\right)
\end{align*}

Let $u \triangleq -\frac{t^2}{2|E|\sigma^2 + \frac{2}{3}t(1-p)}$. Solving for $t$ we obtain:

\[t = \frac{1}{3}u(1-p) + \sqrt{\frac{(1-p)^2 u^2}{9} + 2|E|\sigma^2u}\]

Now we have that:

    \begin{align*}
            P\bigg(\sum L_{u,v} - p \leq \frac{1}{3}u(1-p) + \sqrt{\frac{(1-p)^2 u^2}{9} + 2|E|\sigma^2u}\bigg)
        \geq 1 - e^{-u}
    \end{align*}

We choose $u = \ln\left(\frac{2}{\delta}\right)$, and substituting $|E| = n-1$ for trees and $\sigma^2 = p(1-p)$, we have that with probability $1- \delta$:

\begin{gather*}
    \sum_{(u,v) \in E} \mathbbm{1}(\varphi(u,v)\neq X_{u,v}) \leq 
    \frac{1}{3}\ln(\frac{2}{\delta})(1-p) +
    \sqrt{\frac{(1-p)^2 \ln(\frac{2}{\delta})^2}{9} + 2(n-1)p(1-p)\ln(\frac{2}{\delta})} +
    (n-1)p 
\end{gather*}

Simplifying by noting that $\sqrt{a+b} \leq \sqrt{a} + \sqrt{b}$, we have proven the lemma.

\end{proof}
\subsection{Proof of Lemma 3}
\begin{proof}
For $\hat Y_v = Y_v$, we have that $P_Z( Y_v \neq Z_v) - P_Z(Y_v \neq Z_v) = 0$ and so we are done. When $\hat Y_v \neq Y_v$, we have that $P_Z(Y_v \neq Z_v) = q$ and get the following for the first term:
\begin{align*}
P_Z&(\hat Y_v \neq Z_v)  = P_Z(\hat Y_v \neq Z_v\wedge Z_v = Y_v) + P_Z(\hat Y_v \neq Z_v\wedge Z_v \neq Y_v)  \\ &=  P_Z(\hat Y_v \neq Z_v\wedge Z_v = Y_v) + \sum\limits_{i\in [k]\wedge i\neq \hat{Y}_v \wedge i\neq Y_v} P_Z(\hat Y_v \neq Z_v\wedge Z_v \neq Y_v \wedge Z_v=i)
\end{align*}
We know that $P_Z(\hat Y_v \neq Z_v\wedge Z_v = Y_v)= 1-q$. For each $i$ we have $P_Z(\hat Y_v \neq Z_v\wedge Z_v \neq Y_v \wedge Z_v=i)=\frac{q}{k-1}$. So we have:
\begin{align*}
&\sum\limits_{i\in [k]\wedge i\neq \hat{Y}_v \wedge i\neq Y_v} P_Z(\hat Y_v \neq Z_v\wedge Z_v \neq Y_v \wedge Z_v=i) = \frac{q}{k-1} \sum\limits_{i\in [k]\wedge i\neq \hat{Y}_v \wedge i\neq Y_v} 1 = \frac{q(k-2)}{k-1}
\end{align*}
Given this we have:
\begin{align*}
P_Z(\hat Y_v \neq Z_v)  &=  P_Z(\hat Y_v \neq Z_v\wedge Z_v = Y_v) + \sum\limits_{i\in [k]\wedge i\neq \hat{Y}_v \wedge i\neq Y_v} P_Z(\hat Y_v \neq Z_v\wedge Z_v \neq Y_v \wedge Z_v=i) \\ &= (1-q) +  \frac{q(k-2)}{k-1} = 1-\frac{q}{k-1} 
\end{align*}
Finally, consolidating these, we get:
    \begin{align*}
        P( \hat Y_v \neq Z_v) - P(Y_v \neq Z_v) = 1-\frac{q}{k-1} - q  = 1-\frac{k}{k-1}q 
    \end{align*}
This is exactly $c$, and so the hamming error and excess risk are proportional. Furthermore, we can set $c$ to $1-\frac{k}{k-1}q$.
\end{proof}

\subsection{Proof of Corollary 1}
\begin{proof}
{\color{black}
Before the actual proof, we show that:
\begin{lemma}
In trees, if $Y\in \mathcal{F}$ then $Y^*=Y$.
\end{lemma}
\begin{proof}
We have that $Y\in \mathcal{F}$ hence  $\sum_{(u,v) \in E} \mathbbm{1}\{\varphi( Y_u, Y_v) \neq X_{u,v}\} \leq t$ with $t = (n-1)p + \frac{2}{3}\ln(2/\delta)(1-p) + \sqrt{2(n-1)p(1-p)\ln(2/\delta)} $ and it holds with probability $1-\delta$. We have that $Y^* = \argmin_{Y'\in \mathcal{F}} \sum\limits_{v\in V }\mathbb{P}(Y'_v \neq Y_v)$ which among all possible $Y'\in  \mathcal{F}$ finds the one nearest to $Y$. So we have $Y^*=Y$.
\end{proof}
}
Let $Y^* = \arg \min_{\mathcal{Y} \in \mathcal{F}}\sum_{v\in V} \mathbb{P}(\mathcal{Y}_v\neq Z_v)$ and let $\hat Y$ be the ERM. Because $Y\in \mathcal{F}$ so $Y^*=Y$.Then from Lemma 2, we have:
\begin{gather*}
\sum_{v \in V} P(\hat Y_v \neq Z_v) - \sum_{v \in V}  P(Y^*_v \neq Z_v) \leq \left(\frac{2}{3} + \frac{c}{2}\right)\log\left( \frac{|\mathcal{F}|}{\delta}\right) + \frac{1}{c}\sum_{v \in V} \mathbbm{1}\{\hat Y_v \neq Y_v\}
\end{gather*}
For all $c > 0$, we can use Lemma 3 and apply it to the RHS to obtain
{\color{black} 
\begin{gather*}
    \left(1 - \frac{1}{ct}\right)\bigg[\sum_{v \in V} P(\hat Y_v \neq Z_v) - \sum_{v \in V}  P(Y^*_v \neq Z_v)\bigg] \leq \left(\frac{2}{3} + \frac{c}{2}\right)\log\left( \frac{|\mathcal{F}|}{\delta}\right)
    \end{gather*}
}
where $t = 1-\frac{k}{k-1}q$. Now, because this holds for $c > 0$, we can choose $c =\frac{2}{t}$ thus obtaining
    \begin{gather*}
        \left(\frac{1}{2}\right)\sum_{v \in V} P(\hat Y_v \neq Z_v) - \sum_{v \in V}  P(Y^*_v \neq Z_v) \leq
        \left(\frac{2}{3} + \frac{1}{t}\right)\log\left( \frac{|\mathcal{F}|}{\delta}\right)
    \end{gather*}
Finally, applying that $q= \frac{1}{2} - \varepsilon$, we obtain our result.
\end{proof}

\subsection{Proof of Theorem 1}
\begin{proof}
By Lemma 1, we have  with probability at least $1 - \frac{\delta}{2}$
    \[\sum_{(u,v) \in E} \mathbbm{1}(\varphi(u,v) \neq X_{u,v}) \leq t\]
which we will use it to define a hypothesis class $\mathcal{F}$ as
\[\mathcal{F}= \left\lbrace \hat Y: \sum_{(u,v) \in E} \mathbbm{1}(\varphi(\hat Y_u,\hat Y_v) \neq X_{u,v}) \leq t \right\rbrace\]

with \[t = (n-1)p + \frac{2}{3}\ln(\frac{2}{\delta})(1-p) + \sqrt{2(n-1)p(1-p)\ln(\frac{2}{\delta})}
\]

Which suggests that $Y \in \mathcal{F}$ with high probability. By Corollary 1, we have that $\hat Y$ being the ERM over $\mathcal{F}$ implies that
    \begin{gather*}
        \sum_{v \in V} P(\hat Y_v \neq Z_v) - \min_{Y \in \mathcal{F}} \sum_{v \in V}  P(Y_v \neq Z_v) \leq 
\left(\frac{4}{3} + \frac{2}{\frac{1}{4}+\big( \frac{1}{4}-\epsilon\big)\big( 1-\frac{k}{k-1}\big)}\right)\log\left(\frac{|\mathcal{F}|}{\delta}\right)
    \end{gather*}

Combining this with Lemma 3, we conclude that $\sum_{v \in V} \mathbbm{1}\{\hat Y_v \neq Y_v\}$ is bounded form above by
    \begin{equation*}  \frac{1}{ 1-\frac{k}{k-1}q}  \left(\frac{4}{3} + \frac{2}{\frac{1}{4}+\big( \frac{1}{4}-\epsilon\big)\big( 1-\frac{k}{k-1}\big)}\right)\log\left(\frac{|\mathcal{F}|}{\delta}\right)
    \end{equation*}

Now, we approximate the size of the class $\mathcal{F}$. We can do so by upper-bounding the number of ways to violate the observed measurements. Pessimistically, of the possible $l = 0, 1, \ldots t$ violations, there are at most $l$ nodes which are involved in this violation. Furthermore, there are at most $k-1$ ways for each of these nodes to be involved in such a violation. Therefore, we have,setting $t = \frac{2}{3}\ln(2/\delta)(1-p) + \sqrt{2(n-1)p(1-p)\ln(2/\delta)} +(n-1)p$

\begin{equation*}
\begin{aligned}
|\mathcal{F}| & \leq \sum_{l = 0}^t \binom{n}{l} k^l  \leq k^t \sum_{l = 0}^t \binom{n}{l} \leq k^t 2^t
\end{aligned}
\end{equation*}

Using this bound for $|\mathcal{F}|$, and assuming that the noise and sampling distribution is constant, we obtain that the hamming error is bounded by $\tilde{O}(\log(k)np)$ 
\end{proof}

\subsection{Solving the Optimization Problem on Trees with Dynamic Programming}
Because $G$ is assumed to be a tree, we can compute optimal solutions to subproblems. Specifically, we can turn any undirected tree into a controlled one by a breadth-first search.

Then we can define a table $OPT(u, B | \ell)$ which stores optimal values to the subtree rooted at $u$, constrained to budget $B$ and with the parent of $u$ constrained to class $\ell$. Given the values of $OPT$ for all descendants of a node $u$, it is not difficult to find values for the table at $u$. We formalize this in the following theorem.

\addtocounter{theorem}{2}
\begin{theorem} The optimization problem 1 can be solved in time $O(kn^3p)$.
\end{theorem}
\begin{proof}
Given a tree $T = (V, E)$, a budget t, observations $X = \{X_{u,v}\}_{(u,v) \in E}$ and 
$Z = \{Z_v\}_{v \in V}$, we would like to compute a solution to

\begin{align*}
        \sum_{(u,v) \in E} &\mathbbm{1}(\varphi(Z_u,Z_v) \neq X_{u,v}) \leq
        \frac{2}{3}\ln(2/\delta)(1-p) + \sqrt{2(n-1)p(1-p)\ln(2/\delta)} +(n-1)p
\end{align*}

First, we turn $T$ into a tree rooted at some node $r$ by running a breadth-first search from $r$ and directing nodes according to their time of discovery. Call this directed tree rooted at $r$ $\overrightarrow{T}_r$. We specify a table $OPT$ which will collect values of optimal subproblems. 

Specifically, denote $\overrightarrow{T}_u$ as the subtree of $\overrightarrow{T}_r$ rooted at a node $u$. 
Then OPT will be a matrix parameterized by $OPT(u, B| \ell)$ where $u \in V$, $0 \leq B\leq |\overrightarrow{T}_u|$ (no tree can violate the observations more times than the number of nodes in the tree) and $\ell \in [k]$. Let $Pa(u)$ be the singular parent of the node $u$. Then OPT values represent the optimal value of the subtree rooted at $u$ with a budget $B$ and $Pa(u)$ restricted to the value $\ell$. Our recursive equation for $OPT$ is then

\begin{align}
        &OPT(u, B | i ) = \label{recursive} \min_{\ell \in [k]} \min_{\substack{\sum\limits_{v \in N(u)} B_v \\ =  B - \mathbbm{1}\{X_{u,v} \neq \varphi(i, \ell)\}}} \sum_{v \in N(u)} OPT(v, B_v | \ell)  +  \mathbbm{1}\{\ell \neq Z_u \}  \nonumber
\end{align}

If we have the value of $OPT(u, B | \ell)$ for all nodes $u \neq r$, values $\ell$ and valid budgets $B \leq t$,
we can calculate the optimum value of the tree by the following: We attach a node $r'$ to $r$ by an edge $r' \rightarrow r$ 
and set the information on the node to $X_{r', r} = 1$ then solve $OPT(r, t|1)$, then repeat the process but with $X_{r',r} = -1$, return the smaller of these two values.

For a leaf node $w$, the value of $OPT(w, B'| \ell)$ is simply $\min_{i} \mathbbm{1}\{i \neq Z_w\}$ for $B' = 1$. 
If $B' = 0$ then we must choose $i$ such that it does not violate the side information, i.e. we must have $\varphi(i, \ell) = X_{w, Pa(w)}$

Finally we show how to compute the summation in (\ref{recursive}) efficiently. For each value $\ell \in [k]$ we must optimize the summation
\begin{equation*}
\min_{\substack{\sum\limits_{v \in N(u)} B_v \\ =  B - \mathbbm{1}\{X_{u,v} \neq \varphi(i, \ell)\}}} \left(\sum_{v \in N(u)} OPT(v, B_v | \ell) + \mathbbm{1}\{\ell \neq Z_u \} \right)
\end{equation*}

Because each node's optimal value is independent, we can rewrite this sum by submitting an optional order on $N(u)$ of $1, 2, \ldots, m = |N(u)|$ and reforming this sum to

\begin{gather*}
    \min_{B_1 \in [0, K - \mathbbm{1}\{ \varphi(\ell, s) \neq X_{u, Pa(u)}\}]} OPT(1, B_1| \ell) 
    +   \min_{\sum_{j \in [2,m]} B_j = B - B_1 - \mathbbm{1}\{\varphi(\ell,s) \neq X_{u, Pa(u)}\}} \sum_{j \in 2,m} OPT(j, B_j| \ell)
\end{gather*}

The minimization for the first two vertices whose number of constraints violated are at most $B$ can be solved in $O(B^2)$ time. 
The calculation for the first three vertices can then be done in $O(B^2)$ time by reusing the information from the first two. We can repeat this until we have considered
all children of $u$. Hence because we must calculate this value for all $k$ possible classes, we get an algorithm which takes time $k \sum_{v \in V} |N(v)|B^2 = O(nkB^2)$. 
The statistical analysis below shows that $B$ is $poly(n,p)$.
\end{proof}

\section{ANALYSIS FOR GENERAL GRAPHS}
\subsection{Approximation Correlation Clustering}

We have following Theorem,

\begin{theorem}
\citep{giotis2006correlation} There is a polynomial time factor $0.878$ approximation algorithm for \textsc{MaxAgree[2]} on general graphs. For every $k \geq 3$, there is a polynomial time factor $0.7666$ approximation algorithm for \textsc{MaxAgree[k]} on general graphs.
\end{theorem}

{\color{black}  With this assumption in the worse case, we have labeling with $0.7666$\textsc{Opt[k]}. If $\textsc{Opt}=|E|-b$ which $b$ is the number of bad edges that the optimal does not cover. We know the original graph is a $k$ cluster with no bad-cycle (a cycle with one negative edge), so whatever bad edges that we see are the result of the noise process on the edges, so $b\leq |E|p$ because part of them do not generate bad-cycles. We can consider the approximate process as an extra source to generate more bad edges so we have $|E|-b'\geq \textsc{Approx}[k]=0.7666\textsc{Opt}[k]$. Also, by our assumption we have $p\leq p'$ so $b\leq b'$ 

\begin{align*}
|E|-&b'\geq \textsc{Approx}[k]=0.7666\textsc{Opt}[k] = 0.7666 \big( |E|-b\big) \rightarrow b' \leq 0.2334|E|+0.7666b
\end{align*}
So we have 
\begin{align*}
b\leq b' \leq 0.2334|E|+0.7666b
\end{align*}

We have upper bound for the error introduced by our approximation and we assume all that noise come from edge noise process and the correlation clustering could not correct it, we can assume a noise process with $p'$ such that $b'= |E|p'$ so :

\begin{align*}
& |E|p' = b'  \leq 0.2334|E|+0.7666b \leq 0.2334|E|+0.7666 |E| p \rightarrow p' \leq 0.2334 + 0.7666p
\end{align*}

So we consider exact correlation clustering result in our analyses and if we interested to see the effect of approximation algorithm on the result and get an error bound, we update $p$ to $0.2334 + 0.7666p$ as worst case analysis which means we directly inject the approximation noise error to the results. This assumption is weak because part of $b'$ can be captured by the local and global optimizer which we neglect it.}

\subsection{Proof of Lemma 4}
\begin{proof}
We mix two partitions into one notation and each data point in $D$ shows as $v_i=(\bar Y_i,Z_i)$ , for each $i\in D$, and $\bar Y_i\in \bar Y$ and $Z_i\in Z$. We define $\forall l \in [k]$
\begin{align*}
& X_l = \lbrace v_i|\bar Y_i=l\rbrace \\
& T_l = \lbrace v_i|Z_i=l\rbrace
\end{align*}

and the error is $E = \sum\limits_{v_i\in D} \mathbbm{1}\lbrace Z_{i}\neq \bar Y_i\rbrace$.
The only thing that we allowed to change is the label of $X_l$s. We can represent the partition $X$ and $T$ as,
 \begin{align*}
& X = \lbrace  X_1,X_2,\dots,X_k\rbrace \\
& T = \lbrace T_1,T_2,\dots,T_k\rbrace
\end{align*}
We claim that with Algorithm 2, we can find the permutation $\pi$ on X, such that $E$ minimize. Let $\pi^*$ be the permutation that makes minimum $E$. We prove this theorem with reductio ad absurdum. Therefore
\begin{equation}
\label{eq:red}
E_{\pi^*}\leq E_\pi
\end{equation}

Let $N$ be the set of all $v_i\in D$ such that $\pi(\bar Y_i)\neq \pi^*(\bar Y_i)$, 

\[N=\lbrace v_i\in D|\pi(\bar Y_i)\neq \pi^*(\bar Y_i)\rbrace \] We can write $E$ for $\pi$,

\begin{align*}
 E_\pi&=\sum\limits_{v_i \in D}\mathbbm{1}\lbrace \pi(\bar Y_i)\neq Z_i\rbrace\\&=\sum\limits_{v_i \in N}\mathbbm{1}\lbrace \pi(\bar Y_i)\neq Z_i\rbrace + \sum\limits_{v_i \not\in N}\mathbbm{1}\lbrace \pi(\bar Y_i)\neq Z_i\rbrace
\end{align*}

Similarly we can define $E_{\pi^*}$,

\begin{align*}
E_{\pi^*}&=\sum\limits_{v_i \in D}\mathbbm{1}\lbrace \pi^*(\bar Y_i)\neq Z_i\rbrace\\&=\sum\limits_{v_i \in N}\mathbbm{1}\lbrace \pi^*(\bar Y_i)\neq Z_i\rbrace + \sum\limits_{v_i \not\in N}\mathbbm{1}\lbrace \pi^*(\bar Y_i)\neq Z_i\rbrace
\end{align*}

Second term in $E_{\pi^*}$ and $E_{\pi}$ are equal, using Inequality \ref{eq:red}, and we define $E_\pi(N)=\sum\limits_{v_i\in N}\mathbbm{1}\lbrace \pi(\bar Y_i)\neq Z_i\rbrace $ and similarly $E_{\pi^*}(N)$ for $\pi^*$, so we have,

\begin{equation}
\label{eq:noq}
E_{\pi^*}(N) \leq E_\pi(N)
\end{equation} 

We know $N\subseteq D$, so the partition $X$ on D present a sub-partition $\hat{X}$ on $N$. $\hat{X}$ defines like $X$, so \(\hat{X}=\lbrace \hat{X}_1,\hat{X}_2,\dots,\hat{X}_k\rbrace\). This sub-partition notion can be defined for both permutations $\pi$ and $\pi^*$,
\begin{align*}
& \hat{X}_\pi = \lbrace  \hat{X}_{\pi(1)},\hat{X}_{\pi(2)},\dots,\hat{X}_{\pi(k)}\rbrace \\
& \hat{X}_{\pi^*} = \lbrace \hat{X}_{\pi^*(1)},\hat{X}_{\pi^*(2)},\dots,\hat{X}_{\pi^*(k)}\rbrace
\end{align*}

In the greedy algorithm, we sort the intersections of $X_i$s and $T_i$s and select the biggest one each time, because $\hat{X}$ is sub-partition of $X$, so we have,
\begin{equation}
\label{eq:bi}
\forall v_i,v_j\in \hat{X}\hspace{0.25cm} \pi(\bar Y_i)=\pi(\bar Y_j)\longleftrightarrow \pi^*(\bar Y_i)=\pi^*(\bar Y_j)
\end{equation}
Based on Equation \ref{eq:bi}, we can define a isomorphism on $N$,
\[ \forall v_i\in N\hspace{0.25cm} \phi:\pi^*(\bar Y_i)\rightarrow \pi(\bar Y_i)\]
we define $\dot{\max}()$ as selecting the set with maximum size among all feasible sets, then we have,
\begin{equation}
\label{eq:greedy}
\hat{X}_{\pi(\bar Y_i)} = \nonumber\bigg\lbrace v_j\in D | \pi(\bar Y_i)=\pi(\bar Y_j),\pi(\bar Y_j)\neq Z_j,\dot{\max}  |X_{\bar Y_j}\cap T_{\pi(\bar Y_j)}|\bigg\rbrace
\end{equation}
and also we can obtain,
\begin{align}
\label{eq:err}
 E_{\pi^*}(N)&=\nonumber\sum\limits_{v_i\in N} \mathbbm{1}\lbrace \pi^*(\bar Y_i)\neq Z_i\rbrace \\&=\nonumber \sum\limits_{\hat X_i\in \hat{X}_{\pi^*}}\sum\limits_{v=(\bar Y,Z)\in \hat X_i} \mathbbm{1}\lbrace \pi^*(\bar Y)\neq Z \rbrace\\&=\nonumber\sum\limits_{\hat X_i\in \hat{X}_{\pi^*}}\sum\limits_{v=(\bar Y,Z)\in \hat X_i} \mathbbm{1}\lbrace \phi^{-1}(\pi(\bar Y))\neq Z\rbrace
\end{align}
Also from Equation \ref{eq:greedy}, we know 
\[ \dot{\max}  |X_{\bar Y_j}\cap T_{\pi(\bar Y_i)}|= \hat{X}_{\pi(\bar Y_i)}\cup \underline{X}_{\pi(\bar Y_i)}\] because $ T_{\pi(\bar Y_i)}$ might already given to bigger intersection so we used $\dot{\max}$, and $\underline{X}_{\pi(\bar Y_i)}$ define as,
\begin{align*}
&\underline{X}_{\pi(\bar Y_i)} = \bigg\lbrace v_j\in D | \pi(\bar Y_i)=\pi(\bar Y_j),\pi(\bar Y_j)= Z_j,\dot{\max}  |X_{\bar Y_j}\cap T_{\pi(\bar Y_j)}|\bigg\rbrace
\end{align*} 
Based on greedy $\underline{X}_{\pi(\bar Y_i)}$ is maximized on other hands from Equation \ref{eq:err}, we know that \( E_{\pi^*}\leq E_\pi\), so there exist equivalence $C$ partition based on the $\phi$, we have such that using Inequality \ref{eq:noq},
\begin{equation}
\label{eq:t}
\sum\limits_{v \in C}\mathbbm{1}\lbrace \phi^{-1}(\pi(\bar Y))\neq Z\rbrace \leq \sum\limits_{v \in C}\mathbbm{1}\lbrace \pi(\bar Y)\neq Z\rbrace
\end{equation} 

moreover, this should be true for all $z\in C$. But if $\pi^*(\bar Y)$ is not $\pi(\bar Y)$ then,
\[ \bigg| \bigg\lbrace v_i\in D;\mathbbm{1}\lbrace \pi^*(\bar Y_i)=y_i\rbrace  \bigg\rbrace \bigg|< \underline{X}_{\pi(\bar Y_i)} \]
so this contradicting with Inequality \ref{eq:t} so for equivalence class $C$, we have \[ \sum\limits_{v \in C}\mathbbm{1}\lbrace \pi^*(\bar Y)\neq Z\rbrace = \sum\limits_{v \in C}\mathbbm{1}\lbrace \pi(\bar Y)\neq Z\rbrace \]
and because $\hat{X}_\pi$ and $\hat{X}_{\pi^*}$ is finite, this mean $\phi$ is identity function $\phi(x)=x$ so $\pi=\pi^*$. That mean greedy algorithm finds the best permutation transformations that satisfies $Z$.
\end{proof}

\subsection{Proof of Lemma 5}
\begin{proof}
Let $G=(V,E)$, and set $Y$ is the node labels from $L$ assigned to $V$. Let $C\subseteq L$ be the set of all labels that used in $Y$. The easy case is when we want to change a color $c\in Y$ to $c'\notin Y$, this is like renaming. To proof this lemma, we use induction. For showing an edge, we use $i+j$ means that two end point of nodes have label $i$ and $j$ and the edge label is $+1$. Let $C=\{c,c' \}$, we have multiple scenarios that generate violation $Vi=\lbrace c'+c, c+c', c-c, c'-c',  \rbrace$ and also the set of non-violation scenarios is $nVi=\lbrace c'-c, c-c', c+c, c'+c'  \rbrace$ as you can see $nVi$ and $Vi$ closed under swap operation.

We assume the theorem is true for $|C|=k-1$, let $Y$ used for $k$ colors to color them. We know $k-1$ colors can swap, only color $k$ is matter now, consider swap $i\in[k-1]$ and $k$. All edges involve in this swap is $\lbrace i+k, i-k, k+i, k-i, i+i, i-i, k-k, k+k \rbrace$ and errors involved with these two labels are $\lbrace i+k, k+i, i-i, k-k \rbrace$, and this set size does not change after the swap. 

Based on the statement at the beginning of the proof, we are sure about $k$ appear to $[k-1]$ colors, because it is like renaming, the only thing is changing $k$ to $i$. Let $j$ be a label such that $e=(v_i,v_j)\in E: label(v_l)= j\bigwedge label(v_m)=k$, the number of error are $\lbrace j+k , k+j\rbrace$ and after swap we have same number of edge in this set. So $Y$ and its version after swap, $Y'$ have same number of edge violations on the label set $L$, In other word, for any $L$, we have the following statement. \(\sum_{(u,v) \in E} \mathbbm{1}\{\varphi(Y_u,Y_v) \neq X_{u,v}\} =\\ \sum_{(u,v) \in E} \mathbbm{1}\{\varphi(Y'_u,Y'_v) \neq X_{u,v}\}\)
\end{proof}
\subsection{Proof of Lemma 6}
\begin{proof}
Let $\delta(S)^+$ and $\delta(S)^-$ show the positive and negative edges in $\delta(S)$. We define the external boundary nodes as follow,
\begin{align*}
V^{S} = \lbrace v\in G : (v,e)\in \delta(S) \wedge v\not\in S\rbrace
\end{align*}
and internal boundary nodes as
\begin{align*}
V_{S} = \lbrace v\in G : (v,e)\in \delta(S) \wedge v\in S\rbrace
\end{align*}
It is simple to verify that for each $v\in V^{S}$ there exist $u\in V_{S}$ such that $(u,v)\in \delta(S)$ and vice versa. We know that $\tilde{Y}_v^W = Y_v$ for $v\in V^{S}$. If $\delta(S)^- = \emptyset$ and all edges in $\delta(S)$ be correct, we can follow the labels node in $V^{S}$, so for each $v\in V^{S}$ we select the edges $(v,u)$ in $\delta(S)$  and we define $swap(S,v,u)$ so we have set of mapping $\Phi^+(S)=\lbrace swap(S,v,u): v\in V^{S} \wedge u\in V_{S} \wedge (u,v)\in \delta(S) \rbrace$, from Lemma 5, we know the that the number of violations in $S$ is same, so we resolved some violations in $\delta(S)$ which has contradiction with $\tilde{Y}^W\in \mathbb{I}_{min}$, so when $\delta(S)^- = \emptyset$, at least half of nodes are incorrect and we actually can derive the labeling.

Let $ \Gamma_k(S) $ be all label permutation in $S$ such that each permutation can be represented with a sequence of swaps. We can easily show that any sequence of swap is also does not change the edge violation, so we know for all $\pi\in \Gamma_k(S) $ the number of edge violations in $S$ is constant. Because $ V^{S} $ is correct labeled so at least $ \lceil \frac{\delta(S)}{2}\rceil $ of edges in $ \delta(S) $ are incorrect, otherwise there exist a labeling permutation that contradict with minimization of edge violation because the edges inside $S$ does not add violation but we resolve more than half of $ \delta(S) $, In this case we know the existential of such a this permutation but in binary and $\delta(S)^- = \emptyset$ cases, we can actually build the better permutation.
 
\end{proof}

\subsection{Proof of Lemma 7}
\begin{proof}
From Lemma 6, we know at least half of $\delta(S)$ for any $S\subset W^*$ are incorrect, so we used this to find an upper bound for this probability, so the best permutation of labels also should satisfy  Lemma 6 so we have
\begin{align*}
&\mathbb{P}\bigg(\min\limits_{\pi \in \Gamma_k(W^*)} \mathbbm{1}\lbrace \pi(\bar{Y}^{W^*})\neq Y^W\rbrace > 0\bigg) \leq \sum\limits_{S\subset W^*,S\cap W\neq \emptyset, \bar{S}\cap W\neq \emptyset} p^{\lceil \frac{\delta(S)}{2}\rceil}  \\& \leq \sum\limits_{S\subseteq W^*} p^{\lceil \frac{mincut^*(W)}{2}\rceil} \text{\textit{(because $|\delta(S)|\leq mincut^*(W)$ for all $S\subseteq W^*$})}\\ & \leq 2^{|W^*|}p^{\lceil\frac{mincut^*(W)}{2}\rceil} \hspace{0.2cm} \text{\textit{(there are $2^{|W^*|}$ subsets)}}
\end{align*}
where $mincut^*(W) = min_{S\subset W^*,S\cap W\neq\emptyset,\bar S\cap W^*\neq\emptyset}|\delta_{G(W)}(S)|$
\end{proof}

\subsection{Proof of Lemma 8}
We have following theorem from \citet{boucheron2003concentration}
\begin{theorem}
\label{con}
If there exists a constant $c>0$ such that $V_+ \leq cS$ then 
\begin{align*}
\mathbb{P}\lbrace S\geq \mathbb{E}[S]+t \rbrace \leq exp\bigg( \frac{-t^2}{4c\mathbb{E}[S]+2ct}\bigg)
\end{align*}
Subsequently, with probability at least $1-\delta$,
\begin{align*}
S &\leq \mathbb{E}(S)+\max\bigg\lbrace 4c\log(\frac{1}{\delta}) , 2\sqrt{2c\mathbb{E}(S)\log(\frac{1}{\delta})}\bigg\rbrace\\& \leq 2\mathbb{E}(S)+6c\log(\frac{1}{\delta}).
\end{align*}
\end{theorem}

Now we can prove this theorem,
\begin{proof}
We define a random variable that shows the number of the component that has an error concerning the real labels of each component. This random variable is a function of given edges $X$.

\begin{align}
S(X) = \sum\limits_{W\in \mathcal{W}} \min\limits_{\pi \in \Gamma_k(W)} \mathbbm{1}\lbrace \pi(\tilde{Y}^W(X))\neq Y^W\rbrace \label{S}
\end{align}

We know $S(X)=0$ means perfect matching with a given $X$ and in maximum $S(X)=| \mathcal{W}|$, and also $\tilde{Y}^W(X)$ is the component-wise estimator with given edge labels observation $X$. We know that $S:[k]^{|E|}\rightarrow \mathbb{R}$ so we can use Theorem \ref{con} if we can prove that $S(X)$ satisfies the assumption. 

\begin{align*}
&S(X)-S(X^{(e)})=\sum\limits_{W\in \mathcal{W}}\bigg( \min\limits_{\pi\in \Gamma_k(W)} \mathbbm{1}\bigg[ \pi(\tilde{Y}^W(X))\neq Y^W\bigg] - \min\limits_{\pi\in \Gamma_k(W)} \mathbbm{1}\bigg[ \pi(\tilde{Y}^W(X^{(e)}))\neq Y^W\bigg]  \bigg)
\end{align*} 

The right-hand side of the equation is zero for hypernodes that $e$ is not in them so we can reduce the equation to the hypernodes that have $e$, so we show it with $\mathcal{W}(e)$. Formally $\mathcal{W}(e)=\{W\in \mathcal{W} |e \in E(W)\}$  

\begin{align*}
&S(X)-S(X^{(e)})=\sum\limits_{W\in \mathcal{W}(e)}\bigg( \min\limits_{\pi\in \Gamma_k(W)} \mathbbm{1}\bigg[ \pi(\tilde{Y}^W(X))\neq Y^W\bigg] - \min\limits_{\pi\in \Gamma_k(W)} \mathbbm{1}\bigg[ \pi(\tilde{Y}^W(X^{(e)}))\neq Y^W\bigg]  \bigg)
\end{align*} 

For evaluate Theorem \ref{con}, in next proposition we showed $V_+$ is bounded.

\begin{proposition}
The variation of $V_+$ of $S(X)$ in Equation \ref{S} is bounded, $ V_+\leq cS(X) $.

\end{proposition}
\begin{proof}

 \begin{align*}
(&S(X)-S(X^{(e)}))^2. \mathbbm{1}\bigg(S(X)>S(X)^{(e)}\bigg) =\\ &\mathbbm{1}\bigg(S(X)>S(X)^{(e)}\bigg)\times\sum\limits_{W\in \mathcal{W}(e)}\bigg( \min\limits_{\pi\in \Gamma_k(W)} \mathbbm{1}\bigg[ \pi(\tilde{Y}^W(X))\neq Y^W\bigg] - \min\limits_{\pi\in \Gamma_k(W)} \mathbbm{1}\bigg[ \pi(\tilde{Y}^W(X^{(e)}))\neq Y^W\bigg]  \bigg)^2 \\ & \leq \sum\limits_{W\in \mathcal{W}(e)}\bigg( \min\limits_{\pi\in \Gamma_k(W)} \mathbbm{1}\bigg[ \pi(\tilde{Y}^W(X))\neq Y^W\bigg] \bigg)^2 \\&\text{//second part removed and square of minus part added}\\ & \leq |\mathcal{W}(e)| \sum\limits_{W\in \mathcal{W}(e)} \min\limits_{\pi\in \Gamma_k(W)} \mathbbm{1}\bigg[ \pi(\tilde{Y}^W(X))\neq Y^W\bigg]
\end{align*} 
Now we can use this for calculating the expectation.

We directly start with $V_+$ to find its bound.
\begin{align*}
V_+ &=\sum\limits_{e\in E} \mathbb{E}\bigg[  (S(X)-S(X^{(e)}))^2\cdot \mathbbm{1}\bigg(S(X)>S(X)^{(e)}\bigg) \bigg|X_1,X_2,\dots,X_n \bigg]\\ &= \sum\limits_{e\in E}   (S(X)-S(X^{(e)}))^2\cdot \mathbbm{1}\bigg(S(X)>S(X)^{(e)}\bigg)\times\mathbb{P}\bigg[  (S(X)-S(X^{(e)}))^2\cdot \mathbbm{1}\bigg(S(X)>S(X)^{(e)}\bigg) \bigg|X_1,X_2,\dots,X_n \bigg]\\ & \text{(we assume all probabilities are 1)} \\ & \leq \sum\limits_{e\in E}   (S(X)-S(X^{(e)}))^2. \mathbbm{1}\bigg(S(X)>S(X)^{(e)}\bigg)\\ & \text{//from last result}
\\ & \leq \sum\limits_{e\in E} |\mathcal{W}(e)| \sum\limits_{W\in \mathcal{W}(e)} \min\limits_{\pi\in \Gamma_k(W)} \mathbbm{1}\bigg[ \pi(\tilde{Y}^W(X))\neq Y^W\bigg] \\ & \leq \max\limits_{e\in E}|\mathcal{W}(e)| \sum\limits_{e\in E} \sum\limits_{W\in \mathcal{W}(e)} \min\limits_{\pi\in \Gamma_k(W)} \mathbbm{1}\bigg[ \pi(\tilde{Y}^W(X))\neq Y^W\bigg] \\ & = \max\limits_{e\in E}|\mathcal{W}(e)| \sum\limits_{W\in \mathcal{W}(e)} \sum\limits_{e\in E} \min\limits_{\pi\in \Gamma_k(W)} \mathbbm{1}\bigg[ \pi(\tilde{Y}^W(X))\neq Y^W\bigg] \\ & \leq \max\limits_{e\in E}|\mathcal{W}(e)|\max\limits_{W\in \mathcal{W}}|E(W)| \sum\limits_{W\in \mathcal{W}(e)} \min\limits_{\pi\in \Gamma_k(W)} \mathbbm{1}\bigg[ \pi(\tilde{Y}^W(X))\neq Y^W\bigg] \\ & = \max\limits_{e\in E}|\mathcal{W}(e)|\max\limits_{W\in \mathcal{W}}|E(W)| S(X)
\end{align*}
Therefore, there is $c=\max\limits_{e\in E}|\mathcal{W}(e)|\max\limits_{W\in \mathcal{W}}|E(W)|$ such that $ V_+\leq cS(X) $.

\end{proof}

So with $c=\max\limits_{e\in E}|\mathcal{W}(e)|\max\limits_{W\in \mathcal{W}}|E(W)|$, the Theorem \ref{con} with probability at least $1-\frac{\delta}{2}$ is valid,

\begin{align*}
S \leq 2\mathbb{E}(S)+6\max\limits_{e\in E}|\mathcal{W}(e)|\max\limits_{W\in \mathcal{W}}|E(W)|\log(\frac{2}{\delta})
\end{align*}

We only need to derive $\mathbb{E}(S)$ using Lemma 7, because $\tilde Y \in \mathbb{I}_{\bar Y}$, so Lemma 7 is also valid for $\tilde Y$,

\begin{align*}
\mathbb{E}(S) &= \sum\limits_{W\in \mathcal{W}} \mathbb{P}\bigg( \min\limits_{\pi\in \Gamma_k(W)} \mathbbm{1}\bigg\lbrace \pi(\tilde{Y}^W(X))\neq Y^W\bigg\rbrace \bigg)\times \min\limits_{\pi\in \Gamma_k(W)} \mathbbm{1}\bigg\lbrace \pi(\tilde{Y}^W(X))\neq Y^W\bigg\rbrace \\ & = \sum\limits_{W\in \mathcal{W}} \mathbb{P}\bigg( \min\limits_{\pi\in \Gamma_k(W)} \mathbbm{1}\bigg\lbrace \pi(\tilde{Y}^W(X))\neq Y^W\bigg\rbrace=1 \bigg) \\ & = \sum\limits_{W\in \mathcal{W}} \mathbb{P}\bigg( \min\limits_{\pi\in \Gamma_k(W)} \mathbbm{1}\bigg\lbrace \pi(\tilde{Y}^W(X))\neq Y^W\bigg\rbrace>0 \bigg) \\ &  \leq \sum\limits_{W\in \mathcal{W}} 2^{|W|}p^{\lceil\frac{mincut(W)}{2}\rceil}  \text{  // from Lemma 7}
\end{align*}

so finally we have,

\begin{align*}
&\min\limits_{\pi \in [\Gamma_k]^\mathcal{W}} \sum\limits_{W\in \mathcal{W}} \mathbbm{1}\lbrace \pi(\tilde{Y}^W) \neq Y^W\rbrace \leq\sum\limits_{W\in \mathcal{W}} 2^{|W|+1}p^{\lceil\frac{mincut(W)}{2}\rceil}+6\max\limits_{e\in E}|\mathcal{W}(e)|\max\limits_{W\in \mathcal{W}}|E(W)|\log(\frac{2}{\delta})
\end{align*}

\end{proof}

\subsection{Proof of Theorem 2}
From Lemma 8, we can directly proof same result for extend of tree components.
\addtocounter{corollary}{1}
\begin{corollary}
\label{col}
There is straightforward deduction to derive the result for $W^*=EXT(W)$ on $T=(\mathcal{W},F)$ with probability $1-\frac{\delta}{2}$,
\begin{align*}
&\min\limits_{\pi \in [\Gamma_k]^\mathcal{W}} \sum\limits_{W\in \mathcal{W}} \mathbbm{1}\lbrace \pi(\bar{Y}^{W^*}) \neq Y^{W^*}\rbrace \leq \sum\limits_{W\in \mathcal{W}} 2^{|W^*|+1}p^{\lceil\frac{mincut^*(W)}{2}\rceil}+ 6\max\limits_{e\in E}|\mathcal{W}^*(e)|\max\limits_{W\in \mathcal{W}}|E(W^*)|\log(\frac{2}{\delta}) 
\end{align*}
we define the maximum size of a hyper-graph as its degree  $deg_E^*(T)=\max\limits_{e\in E}|\mathcal{W}^*(e)|$ which $\mathcal{W}^*(e)=\{W\in \mathcal{W} |e \in E(W^*)\}$ and $E(W^*)$ is the set of all edged in $E$ that are in $W^*$, so we have 
\begin{align*}
&\min\limits_{\pi \in [\Gamma_k]^\mathcal{W}} \sum\limits_{W\in \mathcal{W}} \mathbbm{1}\lbrace \pi(\bar{Y}^{W^*}) \neq Y^{W^*}\rbrace \leq 2^{wid^*(W)+2}\sum\limits_{W\in \mathcal{W}}p^{\lceil\frac{mincut^*(W)}{2}\rceil}+6deg_E^*(T)\max\limits_{W\in \mathcal{W}}|E(W^*)|\log(\frac{2}{\delta})
\end{align*}

Where $wid^*(W)\triangleq\max_{W\in \mathcal{W}}|W^*|-1$.

\end{corollary}

Now we can start to Theorem 2, 

\begin{proof}
To prove this theorem, we need to define a hypothesis class and find information bound for the optimal solution in there, next, we can find a bound for the distance of the real answer of the problem and best answer in the hypothesis class.

Consider the following permutation finding of the components in T:

\begin{align*}
\Pi^\star = \argmin\limits_{\Pi\in \Gamma_k^{|\mathcal{W}|}} \sum\limits_{W\in \mathcal{W}} \mathbbm{1}\big\lbrace \Pi(\tilde{Y}^{W}) \neq Y^W\big\rbrace
\end{align*}

from Corollary \ref{col}, we know that 

\begin{align*}
\min\limits_{\pi \in [\Gamma_k]^\mathcal{W}} \sum\limits_{W\in \mathcal{W}} \mathbbm{1}\lbrace \pi(\tilde{Y}^{W}) \neq Y^W\rbrace \leq K_n
\end{align*}

Because $\tilde{Y}^{W^*}$ and $\bar{Y}^{W^*}$ both are in $\mathbb{I}_{\bar{Y}^{W^*}}$ and also $\tilde{Y}^{W}$ is $\tilde{Y}^{W^*}$ restricted to $W$ and $K_n$ is

\begin{align*}
 K_n \triangleq &2^{wid^*(W)+2}\sum\limits_{W\in \mathcal{W}}p^{\lceil\frac{mincut^*(W)}{2}\rceil}+\\&6deg_E^*(T)\max\limits_{W\in \mathcal{W}}|E(W^*)|\log(\frac{2}{\delta})
\end{align*}

So if we have $\Pi^\star $, we can produce a vertex prediction with at most $ K_n$ mistakes with probability $1-\delta$. However, computing $\Pi^\star $ is impossible because we do not have access to $Y$, so we need to see using $Z$ as a noisy version of $Y$, how much approximation error will add to the theoretical bound of prediction.

We define the following hypothesis class, which is defined with $K_n$ so we make even bigger to include an even better possible solution.

\begin{align*}
&\mathcal{F}\triangleq ([k] \times [k])^{\mathcal{W}} \\&\text{~~~s.t.}\sum\limits_{(W,W')\in F} \mathbbm{1}\{\psi(\pi_W,\pi_{W'})\neq S(W,W')\}\leq L_n\big\rbrace
\end{align*}

In this context, each element of $([k] \times [k])^{\mathcal{W}}$ is a vector of size $\mathcal{W}$ element which each sown as $\pi$. Our goal is to show that best permutation is in $\mathcal{F}$ with high probability.

Such that $L_n = deg(T).K_n$ which enrich the hypothesis class with make it bigger than using $K_n$. We know that if $\min\limits_{\pi \in [\Gamma_k(W)]}  \mathbbm{1}\lbrace \pi(\tilde{Y}^{W}) \neq Y^W\rbrace = 0$ for a component $W$ then we can find a $\bar{\pi}_W \in \Gamma_k$ such that we can effect on $Y^W$ to get $\tilde{Y}^W$ so $\bar{\pi}_W(Y^W)= \tilde{Y}^W$. 

We also have 

\begin{align*}
&\sum\limits_{(W,W')\in F} \mathbbm{1}\{\psi(\pi_W,\pi_{W'})\neq S(W,W')\} = \sum\limits_{(W,W')\in F} \mathbbm{1}\{\psi(\pi_W,\pi_{W'})\neq [2.\mathbbm{1}(\tilde{Y}_v^{W},\tilde{Y}_v^{W'})-1]\}
\end{align*}

and we know $v\in W\cap W'$, so if for each $W\in \mathcal{W}$ we have $\bar{\pi}_W$, if the range of $\pi_W$ and $\pi_{W'}$ be same they get $1$ and their range is $Y$, the right hand side also is $1$ because the range of two permutations are $Y^{W}$ and $v\in W\cap W'$, so $\mathbbm{1}\{\psi(\pi_W,\pi_{W'})\neq [2.\mathbbm{1}(\tilde{Y}_v^{W},\tilde{Y}_v^{W'})-1]\} = 0$ when ever $W$ and $W'$ have no errors. Therefore $\Pi^\star\in \mathcal{F}$ with probability $1-\delta$. The complexity of hypothesis class can parametrized with the size of $\mathcal{F}(X)$ so we have

\begin{align*}
|\mathcal{F}(X)|&=\sum\limits_{m=0}^{L_n} {|\mathcal{W}| \choose m} k!^m  \\& \leq \sum\limits_{m=0}^{L_n} {|\mathcal{W}| \choose m} k!^{L_n} =k!^{L_n}  \sum\limits_{m=0}^{L_n} {|\mathcal{W}| \choose m}  \\&\leq k!^{L_n} \bigg( \frac{e|\mathcal{W}|}{L_n}\bigg)^{L_n} \leq  \bigg( \frac{e.n.k!}{L_n}\bigg)^{L_n} 
\end{align*}

We consider non-redundant decomposed trees which means for $(W_i, W_j) \in F$ we have $W_i\backslash (W_i\cap W_j) \neq \emptyset$. In Algorithm 3, we use $Z$ instead of $Y$. So we have 

\begin{align*}
\hat{\pi} = \min\limits_{\pi \in \mathcal{F}(X)} \sum\limits_{W\in \mathcal{W}}\sum\limits_{v\in W} \mathbbm{1}\big\lbrace \pi(\tilde{Y}_v^{W}) \neq Z_v\big\rbrace .
\end{align*}

We have following lemma to continue the proof 
\addtocounter{lemma}{8}
\begin{lemma}
\label{lem:approx}
For $\sum\limits_{v\in W} \mathbbm{1}\big\lbrace \hat{\pi}(\tilde{Y}_v^{W}) \neq \pi^\star(\tilde{Y}_v^{W})\big\rbrace$ we have following approximation,

\begin{align*}
&\sum\limits_{v\in W} \mathbbm{1}\big\lbrace \hat{\pi}(\tilde{Y}_v^{W}) \neq \pi^\star(\tilde{Y}_v^{W})\big\rbrace \\& = \frac{1}{c} \sum\limits_{v\in W\wedge \pi^\star(\tilde{Y}_v^{W}) = Y_v}\bigg\lbrace \mathbb{P}_Z\big\lbrace  \hat{\pi}(\tilde{Y}_v^{W}) \neq Z_v \big\rbrace - \mathbb{P}_Z\big\lbrace   \pi^*_W(\tilde{Y}_v^{W}) \neq Z_v \big\rbrace  \bigg\rbrace_= \\ & + \frac{1}{c'} \sum\limits_{v\in W\wedge \pi^\star(\tilde{Y}_v^{W}) \neq Y_v} \bigg\lbrace \mathbb{P}_Z\big\lbrace  \hat{\pi}(\tilde{Y}_v^{W}) \neq Z_v \big\rbrace - \mathbb{P}_Z\big\lbrace   \pi^*_W(\tilde{Y}_v^{W}) \neq Z_v \big\rbrace  \bigg\rbrace_{\neq}
\end{align*}

such that $c=  -\big( 1-\frac{k}{k-1}q\big)$ and $c' = 1-\frac{k}{k-1}q$.

\end{lemma}
\begin{proof}
We prove this equation step by step 
\begin{align*}
&\sum\limits_{v\in W} \mathbbm{1}\big\lbrace \hat{\pi}(\tilde{Y}_v^{W}) \neq \pi^\star(\tilde{Y}_v^{W})\big\rbrace = \sum\limits_{v\in W\wedge \pi^\star(\tilde{Y}_v^{W}) = Y_v} \mathbbm{1}\big\lbrace \hat{\pi}(\tilde{Y}_v^{W}) \neq \pi^\star(\tilde{Y}_v^{W})\big\rbrace_{=} + \sum\limits_{v\in W\wedge \pi^\star(\tilde{Y}_v^{W}) \neq Y_v} \mathbbm{1}\big\lbrace \hat{\pi}(\tilde{Y}_v^{W}) \neq \pi^\star(\tilde{Y}_v^{W})\big\rbrace_{\neq}\\ &= \frac{1}{c} \sum\limits_{v\in W\wedge \pi^\star(\tilde{Y}_v^{W}) = Y_v}\bigg\lbrace \mathbb{P}_Z\big\lbrace  \hat{\pi}(\tilde{Y}_v^{W}) \neq Z_v \big\rbrace - \mathbb{P}_Z\big\lbrace   \pi^*_W(\tilde{Y}_v^{W}) \neq Z_v \big\rbrace  \bigg\rbrace_= \\ & + \frac{1}{c'} \sum\limits_{v\in W\wedge \pi^\star(\tilde{Y}_v^{W}) \neq Y_v} \bigg\lbrace \mathbb{P}_Z\big\lbrace  \hat{\pi}(\tilde{Y}_v^{W}) \neq Z_v \big\rbrace - \mathbb{P}_Z\big\lbrace   \pi^*_W(\tilde{Y}_v^{W}) \neq Z_v \big\rbrace  \bigg\rbrace_{\neq}
\end{align*}

We have to derive each part of the relation separately, for both sigma if $\hat{\pi}(\tilde{Y}_v^{W}) = \pi^\star(\tilde{Y}_v^{W})$ the above is true for any $c$ and $c'$.

We need to calculate $c$ and $c'$, for $c$ which is $\pi^\star(\tilde{Y}_v^{W}) = Y_v$, we have

\begin{align*}
 \mathbb{P}_Z\big\lbrace  \hat{\pi}(\tilde{Y}_v^{W}) \neq Z_v &\wedge \pi^\star(\tilde{Y}_v^{W}) = Y_v\big\rbrace=  \frac{k-2}{k-1}q 
\end{align*}
and $\mathbb{P}_Z\big\lbrace   \pi^*_W(\tilde{Y}_v^{W}) \neq Z_v \big| \pi^\star(\tilde{Y}_v^{W}) = Y_v \big\rbrace = 1-q$ so we can calculate $c$.

\begin{align*}
& \mathbb{P}_Z\big\lbrace  \hat{\pi}(\tilde{Y}_v^{W}) \neq Z_v \big\rbrace - \mathbb{P}_Z\big\lbrace   \pi^*_W(\tilde{Y}_v^{W}) \neq Z_v \big\rbrace =  \frac{k-2}{k-1}q  - (1-q)=-\big( 1-\frac{k}{k-1}q\big)
\end{align*}
so $c= -\big( 1-\frac{k}{k-1}q\big)$. Next, we calculate $c'$ which is $\pi^\star(\tilde{Y}_v^{W}) \neq Y_v$, therefore we have

\begin{align*}
 &\mathbb{P}_Z\big\lbrace  \hat{\pi}(\tilde{Y}_v^{W}) \neq Z_v \wedge  \pi^\star(\tilde{Y}_v^{W}) \neq Y_v\big\rbrace=  \\& \mathbb{P}_Z\big\lbrace  \hat{\pi}(\tilde{Y}_v^{W}) \neq Z_v \wedge  \pi^\star(\tilde{Y}_v^{W}) \neq Y_v\wedge Y_v=Z_v\big\rbrace + \mathbb{P}_Z\big\lbrace  \hat{\pi}(\tilde{Y}_v^{W}) \neq Z_v \wedge  \pi^\star(\tilde{Y}_v^{W}) \neq Y_v\wedge Y_v\neq Z_v\big\rbrace = \\ &  \frac{k-2}{k-1}q + 1 - \frac{1}{k-1}q = 1- q
\end{align*}

and for second part we have,

\begin{align*}
 &\mathbb{P}_Z\big\lbrace  \pi^\star(\tilde{Y}_v^{W}) \neq Z_v \wedge  \pi^\star(\tilde{Y}_v^{W}) \neq Y_v\big\rbrace=  \\& = \mathbb{P}_Z\big\lbrace  \pi^\star(\tilde{Y}_v^{W}) \neq Z_v \wedge  \pi^\star(\tilde{Y}_v^{W}) \neq Y_v\wedge Y_v=Z_v\big\rbrace  +  \mathbb{P}_Z\big\lbrace  \pi^\star(\tilde{Y}_v^{W}) \neq Z_v \wedge  \pi^\star(\tilde{Y}_v^{W}) \neq Y_v\wedge Y_v\neq Z_v\big\rbrace = \\ & = q + \frac{k-2}{k-1}q
\end{align*}

so we can calculate $c'$

\begin{align*}
 \mathbb{P}_Z\big\lbrace  \hat{\pi}(\tilde{Y}_v^{W}) \neq Z_v \big\rbrace &- \mathbb{P}_Z\big\lbrace   \pi^*_W(\tilde{Y}_v^{W}) \neq Z_v \big\rbrace = \\& =(1-q) - \big[ q + \frac{k-2}{k-1}q \big] \\ &= 1-\frac{k}{k-1}q
\end{align*}

therefore that $c' = 1-\frac{k}{k-1}q$.

\end{proof}

Fix $\hat{\pi} \in \mathcal{F}(X)$ for each component $W\in \mathcal{W}$ we have 

\begin{align*}
&\sum\limits_{v\in W} \mathbbm{1}\big\lbrace \hat{\pi}_W(\tilde{Y}_v^{W}) \neq Y_v\big\rbrace \leq  \sum\limits_{v\in W} \mathbbm{1}\big\lbrace \hat{\pi}(\tilde{Y}_v^{W}) \neq \pi^\star(\tilde{Y}_v^{W})\big\rbrace + \sum\limits_{v\in W} \mathbbm{1}\big\lbrace \pi^\star(\tilde{Y}_v^{W}) \neq Y_v\big\rbrace \text{\hspace{0.5cm} //Triangle inequality} \\ &
\leq \sum\limits_{v\in W} \mathbbm{1}\big\lbrace \hat{\pi}(\tilde{Y}_v^{W}) \neq \pi^\star(\tilde{Y}_v^{W})\big\rbrace + |W| \mathbbm{1}\big\lbrace \pi^\star(\tilde{Y}_v^{W^*}) \neq Y_v\big\rbrace  \text{\hspace{0.5cm} //Maximize component error} \\ &= -\frac{1}{1-\frac{k}{k-1}q} \sum\limits_{v\in W\wedge \pi^\star(\tilde{Y}_v^{W}) = Y_v}\bigg\lbrace \mathbb{P}_Z\big\lbrace  \hat{\pi}(\tilde{Y}_v^{W}) \neq Z_v \big\rbrace - \mathbb{P}_Z\big\lbrace   \pi^*_W(\tilde{Y}_v^{W}) \neq Z_v \big\rbrace  \bigg\rbrace \\ & + \frac{1}{1-\frac{k}{k-1}q} \sum\limits_{v\in W\wedge \pi^\star(\tilde{Y}_v^{W}) \neq Y_v} \bigg\lbrace \mathbb{P}_Z\big\lbrace  \hat{\pi}(\tilde{Y}_v^{W}) \neq Z_v \big\rbrace - \mathbb{P}_Z\big\lbrace   \pi^*_W(\tilde{Y}_v^{W}) \neq Z_v \big\rbrace  \bigg\rbrace + |W| \mathbbm{1}\big\lbrace \pi^\star(\tilde{Y}_v^{W}) \neq Y_v\big\rbrace  \text{\hspace{0.5cm} //From Lemma \ref{lem:approx}} 
\end{align*}

For the first part, we can the following approximation:

\begin{align*}
-&\frac{1}{1-\frac{k}{k-1}q} \sum\limits_{v\in W\wedge \pi^\star(\tilde{Y}_v^{W}) = Y_v}\bigg\lbrace \mathbb{P}_Z\big\lbrace  \hat{\pi}(\tilde{Y}_v^{W}) \neq Z_v \big\rbrace - \mathbb{P}_Z\big\lbrace   \pi^*_W(\tilde{Y}_v^{W}) \neq Z_v \big\rbrace  \bigg\rbrace \\ & \leq 2\sum\limits_{v\in W} \mathbbm{1}\big\lbrace \pi^\star(\tilde{Y}_v^{W}) \neq Y_v\big\rbrace  +\frac{1}{1-\frac{k}{k-1}q} \sum\limits_{v\in W\wedge \pi^\star(\tilde{Y}_v^{W}) = Y_v}\bigg\lbrace \mathbb{P}_Z\big\lbrace  \hat{\pi}(\tilde{Y}_v^{W}) \neq Z_v \big\rbrace -\mathbb{P}_Z\big\lbrace   \pi^*_W(\tilde{Y}_v^{W}) \neq Z_v \big\rbrace  \bigg\rbrace \\ & \leq 2|W| \mathbbm{1}\big\lbrace \pi^\star(\tilde{Y}_v^{W}) \neq Y_v\big\rbrace + \frac{1}{1-\frac{k}{k-1}q} \sum\limits_{v\in W\wedge \pi^\star(\tilde{Y}_v^{W}) = Y_v}\bigg\lbrace \mathbb{P}_Z\big\lbrace  \hat{\pi}(\tilde{Y}_v^{W}) \neq Z_v \big\rbrace - \mathbb{P}_Z\big\lbrace   \pi^*_W(\tilde{Y}_v^{W}) \neq Z_v \big\rbrace  \bigg\rbrace
\end{align*}

We conclude that:

\begin{align*}
&\sum\limits_{v\in W} \mathbbm{1}\big\lbrace \hat{\pi}_W(\tilde{Y}_v^{W}) \neq Y_v\big\rbrace \leq 3|W| \mathbbm{1}\big\lbrace \pi^\star(\tilde{Y}_v^{W}) \neq Y_v\big\rbrace + \\& \frac{1}{1-\frac{k}{k-1}q} \sum\limits_{v\in W\wedge \pi^\star(\tilde{Y}_v^{W}) \neq Y_v} \bigg\lbrace \mathbb{P}_Z\big\lbrace  \hat{\pi}(\tilde{Y}_v^{W}) \neq Z_v \big\rbrace - \mathbb{P}_Z\big\lbrace   \pi^*_W(\tilde{Y}_v^{W}) \neq Z_v \big\rbrace  \bigg\rbrace +\\& \frac{1}{1-\frac{k}{k-1}q} \sum\limits_{v\in W\wedge \pi^\star(\tilde{Y}_v^{W}) = Y_v}\bigg\lbrace \mathbb{P}_Z\big\lbrace  \hat{\pi}(\tilde{Y}_v^{W}) \neq Z_v \big\rbrace - \mathbb{P}_Z\big\lbrace   \pi^*_W(\tilde{Y}_v^{W}) \neq Z_v \big\rbrace  \bigg\rbrace \\& \leq  3|W| \mathbbm{1}\big\lbrace \pi^\star(\tilde{Y}_v^{W}) \neq Y_v\big\rbrace +\\& \frac{1}{1-\frac{k}{k-1}q} \sum\limits_{v\in W\wedge \pi^\star(\tilde{Y}_v^{W}) \neq Y_v} \bigg\lbrace \mathbb{P}_Z\big\lbrace  \hat{\pi}(\tilde{Y}_v^{W}) \neq Z_v \big\rbrace - \mathbb{P}_Z\big\lbrace   \pi^*_W(\tilde{Y}_v^{W}) \neq Z_v \big\rbrace  \bigg\rbrace +\\& \frac{1}{1-\frac{k}{k-1}q} \sum\limits_{v\in W\wedge \pi^\star(\tilde{Y}_v^{W}) = Y_v}\bigg\lbrace \mathbb{P}_Z\big\lbrace  \hat{\pi}(\tilde{Y}_v^{W}) \neq Z_v \big\rbrace - \mathbb{P}_Z\big\lbrace   \pi^*_W(\tilde{Y}_v^{W}) \neq Z_v \big\rbrace  \bigg\rbrace  \\& \leq  3|W| \mathbbm{1}\big\lbrace \pi^\star(\tilde{Y}_v^{W}) \neq Y_v\big\rbrace + \frac{1}{1-\frac{k}{k-1}q}  \sum\limits_{v\in W}\bigg\lbrace \mathbb{P}_Z\big\lbrace  \hat{\pi}(\tilde{Y}_v^{W}) \neq Z_v \big\rbrace - \mathbb{P}_Z\big\lbrace   \pi^*_W(\tilde{Y}_v^{W}) \neq Z_v \big\rbrace  \bigg\rbrace 
\end{align*}

We apply this formula for all components $W\in \mathcal{W}$ we have

\begin{align*}
&\sum\limits_{W\in \mathcal{W}}\sum\limits_{v\in W} \mathbbm{1}\big\lbrace \hat{\pi}_W(\tilde{Y}_v^{W}) \neq Y_v\big\rbrace  \\ &\leq 3 \bigg(\max\limits_{W\in \mathcal{W}} |W|\bigg) \sum\limits_{W\in \mathcal{W}}\mathbbm{1}\big\lbrace \pi^\star(\tilde{Y}_v^{W}) \neq Y_v\big\rbrace + \frac{1}{1-\frac{k}{k-1}q}  \sum\limits_{W\in \mathcal{W}}\sum\limits_{v\in W}\bigg\lbrace \mathbb{P}_Z\big\lbrace  \hat{\pi}(\tilde{Y}_v^{W}) \neq Z_v \big\rbrace - \mathbb{P}_Z\big\lbrace   \pi^*_W(\tilde{Y}_v^{W}) \neq Z_v \big\rbrace  \bigg\rbrace \\ & \leq  3 \bigg(\max\limits_{W\in \mathcal{W}} |W|\bigg) K_n + \frac{1}{1-\frac{k}{k-1}q}  \sum\limits_{W\in \mathcal{W}}\sum\limits_{v\in W}\bigg\lbrace \mathbb{P}_Z\big\lbrace  \hat{\pi}(\tilde{Y}_v^{W}) \neq Z_v \big\rbrace - \mathbb{P}_Z\big\lbrace   \pi^*_W(\tilde{Y}_v^{W}) \neq Z_v \big\rbrace  \bigg\rbrace
\end{align*}

using Lemma 2 for right hand side of the equation, we have excess risk bound with probability $1-\frac{\delta}{2}$,

\begin{align*}
& \sum\limits_{W\in \mathcal{W}}\sum\limits_{v\in W}\bigg\lbrace \mathbb{P}_Z\big\lbrace  \hat{\pi}(\tilde{Y}_v^{W}) \neq Z_v \big\rbrace - \mathbb{P}_Z\big\lbrace   \pi^*_W(\tilde{Y}_v^{W}) \neq Z_v \big\rbrace  \bigg\rbrace \\ &\leq \bigg( \frac{2}{3}+\frac{c}{2}\bigg)\log(\frac{2|\mathcal{F}(X)|}{\delta})+\frac{1}{c}\sum\limits_{W\in \mathcal{W}}\sum\limits_{v\in W}\mathbbm{1}\big\lbrace  \hat\pi_W(\tilde Y_v^{W})\neq Y_v \big\rbrace
\end{align*}

so we can mix these inequalities, 

\begin{align*}
& \sum\limits_{W\in \mathcal{W}}\sum\limits_{v\in W} \mathbbm{1}\big\lbrace \hat{\pi}_W(\tilde{Y}_v^{W}) \neq Y_v\big\rbrace \\ & \leq  3 \bigg(\max\limits_{W\in \mathcal{W}} |W|\bigg) K_n + \frac{1}{1-\frac{k}{k-1}q}  \bigg( \frac{2}{3}+\frac{c}{2}\bigg)\log(\frac{2|\mathcal{F}(X)|}{\delta})+\frac{1}{c}\sum\limits_{W\in \mathcal{W}}\sum\limits_{v\in W}\mathbbm{1}\big\lbrace  \hat\pi_W(\tilde Y_v^{W})\neq Y_v \big\rbrace
\end{align*}

so we have 

\begin{align*}
&\sum\limits_{W\in \mathcal{W}}\sum\limits_{v\in W} \mathbbm{1}\big\lbrace \hat{\pi}_W(\tilde{Y}_v^{W}) \neq Y_v\big\rbrace \leq \frac{1}{1-\frac{1}{c}} \bigg[\bigg(3\max\limits_{W\in \mathcal{W}} |W|\bigg) K_n + \frac{1}{1-\frac{k}{k-1}q}  \bigg( \frac{2}{3}+\frac{c}{2}\bigg)\log(\frac{2|\mathcal{F}(X)|}{\delta})\bigg]
\end{align*}

We put $c=\frac{1}{1-\epsilon}$ and rearrange then with probability $1-\delta$ we have
\begin{align*}
&\sum\limits_{W\in \mathcal{W}}\sum\limits_{v\in W} \mathbbm{1}\big\lbrace \hat{\pi}_W(\tilde{Y}_v^{W}) \neq Y_v\big\rbrace \\& \leq \frac{1}{1-\frac{1}{\frac{1}{1-\epsilon}}} \bigg[\bigg(3\max\limits_{W\in \mathcal{W}} |W|\bigg) K_n + \frac{1}{1-\frac{k}{k-1}q}  \bigg( \frac{2}{3}+\frac{\frac{1}{1-\epsilon}}{2}\bigg)\log(\frac{2|\mathcal{F}(X)|}{\delta})\bigg] \\ & = \frac{1}{\epsilon} \bigg[\bigg(3\max\limits_{W\in \mathcal{W}} |W|\bigg) K_n + \frac{1}{1-\frac{k}{k-1}q}  \bigg( \frac{2}{3}+\frac{1}{2(1-\epsilon)}\bigg)\log(\frac{2|\mathcal{F}(X)|}{\delta})\bigg]
\end{align*}

From before, we have $|\mathcal{F}(X)| \leq  \big( \frac{en.k!}{L_n}\big)^{L_n}$, $wid(T) = \max\limits_{W\in \mathcal{W}} |W| $, $K_n$, and Lemma 2 so we can conclude 

\begin{align*}
&\sum\limits_{W\in \mathcal{W}}\sum\limits_{v\in W} \mathbbm{1}\big\lbrace \hat{\pi}_W(\tilde{Y}_v^{W}) \neq Y_v\big\rbrace = \\ & =\frac{1}{\epsilon} \bigg(3\max\limits_{W\in \mathcal{W}} |W|\bigg) K_n + \frac{1}{\epsilon.\big( 1-\frac{k}{k-1}q\big)}  \bigg( \frac{2}{3}+\frac{1}{2(1-\epsilon)}\bigg)\log(\frac{2|\mathcal{F}(X)|}{\delta})\\ & = \frac{3}{\epsilon}.wid(T). K_n + \frac{1}{\epsilon.\big( 1-\frac{k}{k-1}q\big)}  \bigg( \frac{2}{3}+\frac{1}{2(1-\epsilon)}\bigg)\times\big(\log(\frac{2}{\delta})+L_n . \log( \frac{en.k!}{L_n}\big))\big)  \\ & \leq  \frac{3}{\epsilon}.wid(T). K_n + \frac{1}{\epsilon.\big( 1-\frac{k}{k-1}q\big)}  \bigg( \frac{2}{3}+\frac{1}{2(1-\epsilon)}\bigg)\times\big[\log(\frac{2}{\delta})+K_n.deg(T) . k.\log(n.k) \big]  \\& = 
 \frac{1}{\epsilon}. K_n\times \big[ 3.wid(T) +deg(T) . k.\log(n.k).\frac{1}{1-\frac{k}{k-1}q} .(\frac{2}{3}+\frac{1}{2(1-\epsilon)}) \big] + \frac{1}{\epsilon.\big( 1-\frac{k}{k-1}q\big)}  \bigg( \frac{2}{3}+\frac{1}{2(1-\epsilon)}\bigg)\log(\frac{2}{\delta})  \\& \leq \frac{1}{\epsilon}. K_n\times \big[ 3.wid(T) +deg(T) . k.\log(n.k).\frac{1}{1-\frac{k}{k-1}q} .(\frac{2}{3}+\frac{1}{2(1-\epsilon)}) \big] + \frac{1}{\epsilon.\big( 1-\frac{k}{k-1}q\big)}  \bigg( \frac{2}{3}+\frac{1}{2(1-\epsilon)}\bigg)\log(\frac{2}{\delta}) \\ &\leq  \frac{1}{\epsilon}.\bigg[ 2^{wid^*(W)+2}\sum\limits_{W\in \mathcal{W}}p^{\lceil\frac{mincut^*(W)}{2}\rceil}+6deg_E^*(T)\max\limits_{W\in \mathcal{W}}|E(W^*)|\log(\frac{2}{\delta}) \bigg] \\&\times \big[ 3.wid(T) + deg(T) . k.\log(n.k).\frac{1}{1-\frac{k}{k-1}q} .(\frac{2}{3}+\frac{1}{2(1-\epsilon)}) \big] + \frac{1}{\epsilon.\big( 1-\frac{k}{k-1}q\big)}  \bigg( \frac{2}{3}+\frac{1}{2(1-\epsilon)}\bigg)\log(\frac{2}{\delta}) 
\end{align*}

so we have 
\begin{align*}
&\sum\limits_{W\in \mathcal{W}}\sum\limits_{v\in W} \mathbbm{1}\big\lbrace \hat{\pi}_W(\tilde{Y}_v^{W}) \neq Y_v\big\rbrace \\ &\leq O\bigg(
 \frac{1}{\epsilon^2}.\bigg[ 2^{wid^*(W)+2}\sum\limits_{W\in \mathcal{W}}p^{\lceil\frac{mincut^*(W)}{2}\rceil}+6deg_E^*(T)\max\limits_{W\in \mathcal{W}}|E(W^*)|\log(\frac{2}{\delta}) \bigg]\times \big[ 3.wid(T) + deg(T) . k.\log(n.k)\big] \bigg)\\&\text{because mincut  $\geq$ maximum degree} \\& \leq\tilde{O}\big( k.\log k.p^{\lceil\frac{\Delta}{2}\rceil}.n\big)
\end{align*}
As $\hat{\pi}_W(\tilde{Y}_v)= \hat{Y}_v$, so the algorithm ensures Hamming error has driven upper bound.
\end{proof}
{\color{black} 
\section{MIXTURE OF EDGES AND NODES INFORMATION}

In all previous works \citep{foster,meshi,globerson}, the algorithms consider the information of edge and node labels in different stages. For instance in \citep{globerson}, first solves the problem based on the edge because $p<q$, then it uses the nodes information. The information value of positive and negative edges in binary cases are same, but this courtesy breaks under categorical labels, on the other hand, we can use some properties in the graph to trust more on some information. We can calculate the probability of correctness of graph nodes and edges label using $p$ and $q$. In categorical labeling,  the space of noise has some variations from the binary case, so we have the following facts in the categorical case:
\begin{itemize}
\item Flipping an edge makes an error.
\item Switching the label of a node might not make an error.
\end{itemize}

Using Bayes rule and the property of nodes, we have $Pr(v=i|v^\prime=j) =Pr(v^\prime=j|v=i))$, the prim for a vertex shows the vertex after effecting noise. 

We have following theorem the proof come in supplementary material, 

\begin{theorem}
\label{trust}
The likelihood of correctness of an edge $e=(v_i,v_j)\in E$ with label with $L$ are as follow,  
    \begin{align*}
Pr&(L \text{ is untouched }| e, L) =  \\ &   c_L\times\begin{cases}
       2(1-q)q+(\frac{q}{k-1})^2 . \frac{k-2}{k.(k-1)} &L=1,\textit{vio}\\
        (1-q)^2 + (\frac{q}{k-1})^2 . \frac{1}{k.(k-1)} &L=1, \textit{nvio}\\
       2(1-q).\frac{q}{k-1} + (\frac{q}{k-1})^2 . \frac{k-2}{k(k-1)} &L=-1,\textit{vio} \\
          (1-q)^2 + (\frac{q}{k-1})^2 . \frac{k-2}{k} &L=-1, \textit{nvio}\\
     \end{cases}
    \end{align*}
which $c_L= \frac{(1-p)|E|}{\# L \text{ in graph}}$, \textit{vio} means $\phi(X_i,X_j)\neq X_{ij}$, and \textit{nvio} means $\phi(X_i,X_j)= X_{ij}$.
\end{theorem}

\begin{proof}
In all cases, two head nodes of a given edge are $v_i$ and $v_j$, and $L$ shows the label of the edge. We first calculate the probability $Pr(v_i,v_j,L|L \text{ is untouched})$ the using Bayes theorem, we derive the likelihood.
\begin{itemize}
\item The first case is $e$ generates a violation $\phi(Z_i, Z_j)\neq X_{ij}$, and the edge label $L=1$, in this case, the probability of the event is only one of the node labels are changed or both node labels have been changed but to the different labels.
\begin{align*}
& Pr(\text{only one of the node labels are changed}) = \\& 2 Pr(v_i \text{ is changed}) = \\ &2(1-q).\sum\limits_{v_i.label=j \wedge  j\neq X_i} Pr(v_i.label=j | v_i.label=i) \\ &= 2(1-q).\sum\limits_{v_i.label=j \wedge  j\neq X_i} \frac{q}{k-1} = 2.(1-q)q
\end{align*}
and also we have, ($v'_i$ and $ v'_j$ are the label of given nodes after noise effect)
\begin{align*}
 &Pr(v'_i \neq v'_j\wedge v_i =v_j\wedge v'_j\neq v_j\wedge v'_i\neq v_i) \\ &= Pr(v_i\neq v'_i).Pr(v_j\neq v'_j).Pr(v_i =v_j)\times Pr(v'_i \neq v'_j|v_i =v_j\wedge v'_j\neq v_j\wedge v'_i\neq v_i)\\  &=  \frac{q}{k-1} . \frac{q}{k-1} .\frac{1}{k}. \frac{(k-1)(k-2)}{(k-1).(k-1)}\\& = (\frac{q}{k-1})^2 . \frac{k-2}{k.(k-1)}
\end{align*}
Because $Pr(v_i\neq v'_i)$, $Pr(v_j\neq v'_j)$, and $Pr(v_i =v_j)$ are independent, so the whole probability would be $2.(1-q)q+(\frac{q}{k-1})^2 . \frac{k-2}{k.(k-1)}$ .

\item The second case is $e$ does not generate any violation, $\phi(Z_i, Z_j)= X_{ij}$, and the edge label $L=1$, in this case, either both node labels are untouched or they changed but to the same label.
\begin{align*}
 &Pr(\text{both node labels are untouched}) =\\& Pr(v_i = v'_i).Pr(v_j = v'_j)=(1-q)(1-q)=(1-q)^2
\end{align*}
and also we have,
\begin{align*}
& Pr(v'_i = v'_j\wedge v_i \neq v'_i \wedge v_j\neq v'_j\wedge v_i=v_j ) \\ &= Pr(v_i\neq v'_i).Pr(v_j\neq v'_j).Pr(v_i=v_j)\times Pr(v'_i = v'_j|v_i \neq v'_i \wedge v_j\neq v'_j\wedge v_i=v_j)\\  &=  \frac{q}{k-1} . \frac{q}{k-1}.\frac{1}{k} . \frac{(k-1)(1)}{(k-1).(k-1)}  \\&= (\frac{q}{k-1})^2 . \frac{1}{k.(k-1)}
\end{align*}
so the whole probability would be $(1-q)^2 + (\frac{q}{k-1})^2 . \frac{1}{k.(k-1)}$ .

\item The third case is $e$ generates a violation $\phi(Z_i, Z_j)\neq X_{ij}$, and the edge label $L=-1$, in this case, the probability of the event is either one label change to the same label of other head or both change to the same label
\begin{align*}
&Pr(\text{a label change to the same of other head}) \\&= 2 Pr(v_i \text{ is changed to } X_j) = 2(1-q).\frac{q}{k-1}
\end{align*}
and also we have,
\begin{align*}
& Pr(v'_i = v'_j\wedge v_i \neq v'_i \wedge v_j\neq v'_j\wedge v_i\neq v_j) \\& = Pr(v_i\neq v'_i).Pr(v_j\neq v'_j).Pr( v_i\neq v_j) \times Pr(v'_i = v'_j |v_i \neq v'_i \wedge v_j\neq v'_j\wedge v_i\neq v_j) \\  &=  \frac{q}{k-1} . \frac{q}{k-1} . \frac{k-1}{k} . \frac{(k-2)(1)}{(k-1).(k-1)} \\&=(\frac{q}{k-1})^2 . \frac{k-2}{k(k-1)}
\end{align*}
so the whole probability would be $2(1-q).\frac{q}{k-1} + (\frac{q}{k-1})^2 . \frac{k-2}{k(k-1)}$ .

\item The fourth case is $e$ does not generate any violation, $\phi(Z_i, Z_j)= X_{ij}$, and the edge label $L=-1$, in this case, either both node labels are untouched or they changed but to different labels.
\begin{align*}
 &Pr(\text{both node labels are untouched}) \\&= Pr(v_i = v'_i).Pr(v_j = v'_j)\\&=(1-q)(1-q)=(1-q)^2
\end{align*}
and also we have,
\begin{align*}
 &Pr(v'_i \neq v'_j\wedge v_i \neq v'_i \wedge v_j\neq v'_j\wedge v_i \neq v_j ) \\ &= Pr(v_i\neq v'_i).Pr(v_j\neq v'_j).Pr(v_i\neq v_j)\times Pr(v'_i \neq v'_j|v_i\neq v'_i \wedge v_j\neq v'_j\wedge v_i \neq v_j )\\  &=  \frac{q}{k-1} . \frac{q}{k-1}.  \frac{k-1}{k} . \frac{(k-1)(k-2)}{(k-1).(k-1)} \\&= (\frac{q}{k-1})^2 . \frac{k-2}{k}
\end{align*}
so the whole probability would be $(1-q)^2 + (\frac{q}{k-1})^2 . \frac{k-2}{k}$ .
\end{itemize}
Based on the Bayes theorem we have,
\begin{align*}
&Pr(L \text{ is untouched}| v_i,v_j,L) =  \frac{Pr(v_i,v_j,L|L \text{ is untouched}).Pr(L \text{ is untouched})}{Pr(v_i,v_j,L)}
\end{align*}
We have $Pr(v_i,v_j,L)=\frac{\# L \text{ in graph}}{|E|}$, and $Pr(L \text{ is untouched})=1-p$, so we can derive the result.
\end{proof}

As it can be seen with $k=2$, the trust score for positive and negative are only depend to their frequencies, and if their frequencies are equal we can trust them equally. 

\begin{example} (Uniform Frequencies)
Let $\# \lbrace L=+1\rbrace \simeq \#\lbrace L=-1\rbrace$ and $k\geq 3$, then the second part of is negligible because of $(\frac{q}{k-1})^2$ parameter, then if $2(1-q)q\leq (1-q)^2$ and $2(1-q).\frac{q}{k-1}\leq (1-q)^2$ which is $q<\min\{ \frac{1}{3}, \frac{k-1}{k+1}\}=\frac{1}{3}$ then the non-violating edges are more reliable.
\end{example}

The following example is more related to the grid graphs that considered in \citep{globerson}.

\begin{example} (Image Segmentation)
The case $k\geq 3$ and $\# \lbrace L=+1\rbrace \geq \#\lbrace L=-1\rbrace$, which we usually see in the images, because the negative edges are on the boundary of regions. If $q<1/3$, We have can trust more on the non-violating negative edges than non-violating positive edges. 
\end{example}

To the best of our knowledge, no algorithm considers the mixture of edges and nodes information on the categorical data. Therefore, Theorem \ref{trust} can be a guide to design such an algorithm.
}

\section{EXPERIMENT RESULTS}
\subsection{Details on Experimental Setup}
We provide a detailed discussion on our experimental setup.

\paragraph{Trees Generation Process:} We generate random trees, and we apply the noise to the generated graph. We need to have at least one example of each $k$ labels, so the generation process starts by creating $k$ nodes, one example for each category. Then, it generates $k$ random numbers $n_1,\dots,n_k$ such that $ \sum_{i=1}^k n_i = n-k$. Next, it creates tree edges for the set of nodes $V$. Let $S$ and $E$ be empty sets. We select two nodes $v$ and $u$ randomly from $V$ and add $(u,v)$ to $E$ such that the label of the edge satisfies the label of $u$ and $v$ and set $S=S\cup\{u,v\}$, and $V=V\backslash\{u,v\}$. Now, we select one node $v\in S$ and one node $u\in V$ randomly and add $(u,v)$ to $E$ such that the edge label satisfies the endpoints and remove $u$ from $V$ and add it to $S$. We repeat until $V$ is empty. This process follows the Brooks theorem \citep{brooks_1941}. Finally, we apply uniform noise model with probabilities of $p$ and $q$. We select this simple generative process because it covers an extensive range of random trees.

\paragraph{Grids Graph Generation:} We use gray scale images as the source of grid graphs. The range of pixel values in gray scale images is $r=[0,255]$, so we have that $0\leq k\leq 255$. We divide $r$ to $k$ equal ranges $\{r_1,r_2,\dots,r_k \}$. We map all pixels whose values are in $r_i$ to $median(r_i)$. For edges, we only consider horizontal and vertical pixels and assign the ground truth edge labels based on the end points. We generate noisy node and edge observations using the uniform noise model. We use \citet{griffin2007caltech} dataset to select gray-scale images. 

\paragraph{Baseline Method:} A Majority Vote Algorithm: For each node $v\in G$ assign $f_v=[s_1,s_2,\dots,s_k]$ with $s_i=0: \forall i\in [k]$. Let $label(.)$ shows the label of the passed node. Then, for nodes in neighbourhood of $v$, $u\in N(v)$, we update $f_v$ with $s_{label(u)}=s_{label(u)}+X_{uv}$. At the end, for each node $v$, $\hat Y_v=\arg\max_{i\in [k]}(f_v)$ if $|\max(f_v)|=1$ otherwise if $Z_v\in\arg \max(f_v)$, then $\hat Y_v=Z_v$ else $\hat Y_v = random(\arg\max(f_v))$. This is a simple baseline. We use it as we want to validate that our methods considerably outperform simple baselines.

\paragraph{Evaluation Metric:} We use the normalized Hamming distance $ \sum_{v\in V} \mathbbm{1}(Y_v\neq \hat Y_v)/|V|$. between an estimated labeling $\hat Y$ and the ground truth labeling $Y$.

\subsection{Additional Experiments on Grids}

We provide some qualitative results on the performance of our methods.

Figure \ref{WithErrorAndMajority} presents a qualitative view of the results obtained by our method (and the majority vote baseline) as $k$ increases on the grey scale images. We see that using only the edge information (edge-based prediction) becomes more chaotic for larger values of $k$. This is because the information that edges carry decreases. However, we see that combining the information provided by both node and edge observations allows us to recover the noisy image. As expected, the simple Majority vote baseline yields worse results than our method.

\begin{figure*}[!h]
\vskip 0.2in
\begin{center}
\centerline{\includegraphics[width=\textwidth,trim=4 4 4 6,clip]{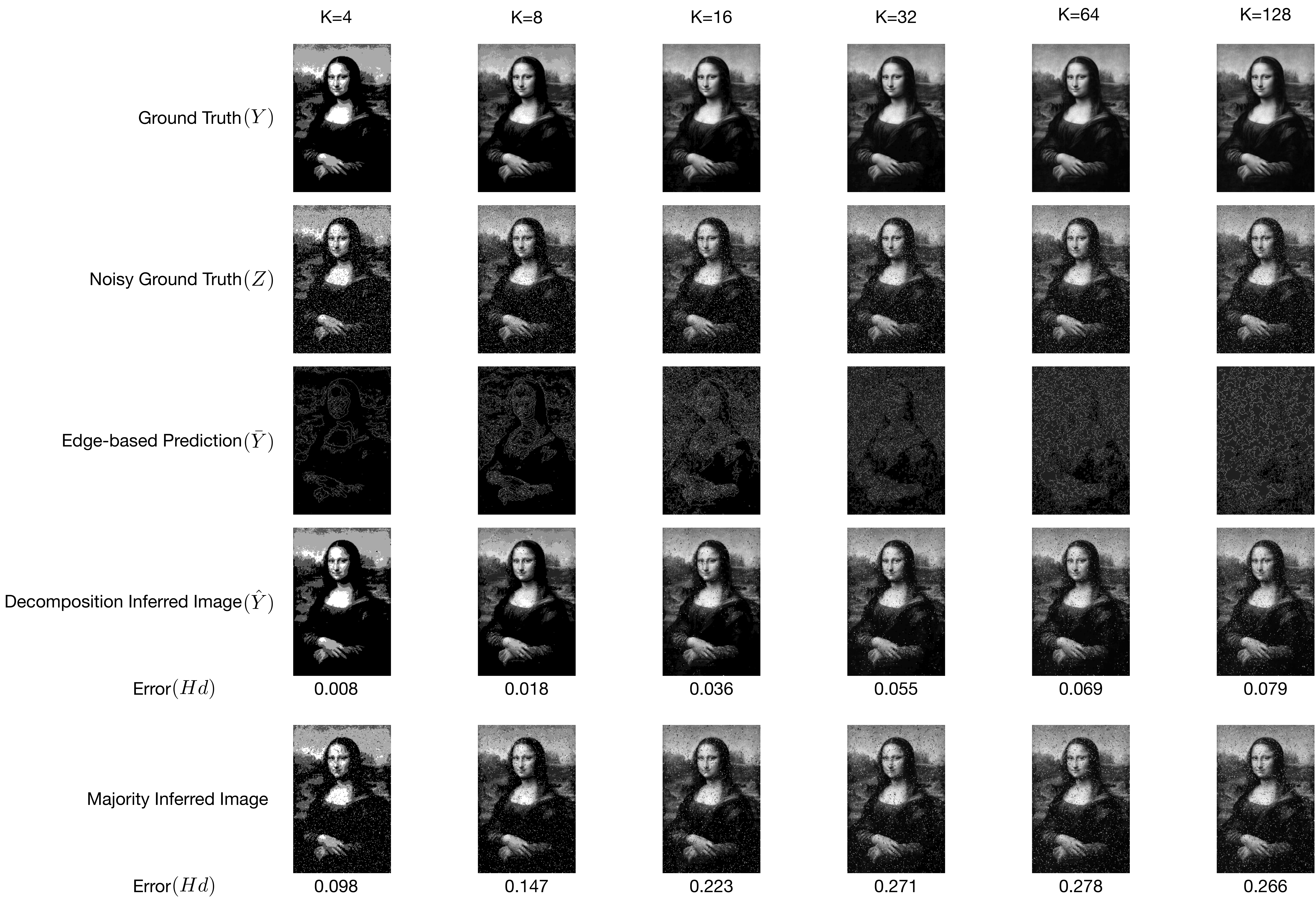}}
\caption{At each column, different stages of the inference process on the image that generates median error can be seen. It starts with generating $k$ value image, adding noise following the model, generates best edge based prediction, and minimize it with noisy ground truth; we also report its corresponding error, you can also see the result and its error from majority algorithm.}
\label{WithErrorAndMajority}
\end{center}
\vskip -0.2in
\end{figure*}

\newpage
\section{REFERENCES}

\bibliography{long}
\bibliographystyle{icml2019}

\end{document}